\documentclass[lettersize,journal]{IEEEtran}
\usepackage{graphicx}
\usepackage{subcaption}
\usepackage{cite}  
\usepackage{amsmath,amsfonts}
\usepackage{algorithmic}
\usepackage{array}
\usepackage[caption=false,font=normalsize,labelfont=sf,textfont=sf]{subfig}
\usepackage{textcomp}
\usepackage{stfloats}
\usepackage{url}
\usepackage{verbatim}
\usepackage{amsmath}
\usepackage{amssymb}
\usepackage{amsfonts}
\usepackage{mathtools}
\usepackage{amsthm}

\usepackage{pgfplots}
\usepackage{booktabs}
\usepackage{algorithm}
\usepackage{algorithmic}
\usepackage{listings}
\usepackage{multirow}

\usepackage{pdfpages}
\usepackage{tabularx}
\usepackage{times}
\usepackage{url}
\usepackage{cite}
\usepackage{microtype}

\usepackage{tabularx}
\usepackage{siunitx}
\theoremstyle{plain}
\newtheorem{theorem}{Theorem}
\newtheorem{definition}{Definition}

\newtheorem{lemma}[theorem]{Lemma}

\theoremstyle{remark}

\def\BibTeX{{\rm B\kern-.05em{\sc i\kern-.025em b}\kern-.08em
    T\kern-.1667em\lower.7ex\hbox{E}\kern-.125emX}}
\usepackage{balance}
\begin{document}
\title{Inductive Convolution Nuclear Norm Minimization for  \\
           Tensor Completion with Arbitrary Sampling}
\author{Wei Li, ~\IEEEmembership{Student Member,~IEEE}, Yuyang Li,~\IEEEmembership{Student Member,~IEEE}, Kaile Du,~\IEEEmembership{Student Member,~IEEE}, Yi Yu,~\IEEEmembership{Student Member,~IEEE}, and Guangcan Liu~\IEEEmembership{Senior Member,~IEEE}
\thanks{Manuscript received .}
\IEEEcompsocitemizethanks{
The authors are with the School of Automation, Southeast University, Nanjing, China 210018. (e-mail: liwei1999@seu.edu.cn,yuyangli@seu.edu.cn,  kailedu@seu.edu.cn, yuyi2001@seu.edu.cn, guangcanliu@seu.edu.cn).
}
\thanks{Copyright}
}

\markboth{Journal of \LaTeX\ Class Files,~Vol.~18, No.~9, September~2020}%
{How to Use the IEEEtran \LaTeX \ Templates}

\maketitle

\begin{abstract}
The recently established Convolution Nuclear Norm Minimization (CNNM) addresses the problem of \textit{tensor completion with arbitrary sampling} (TCAS), which involves restoring a tensor from a subset of its entries sampled in an arbitrary manner. Despite its promising performance, the optimization procedure of CNNM needs performing Singular Value Decomposition (SVD) multiple times, which is computationally expensive and hard to parallelize. To address the issue, we reformulate the optimization objective of CNNM from the perspective of convolution eigenvectors. By introducing pre-learned convolution eigenvectors which are shared among different tensors, we propose a novel method called Inductive Convolution Nuclear Norm Minimization (ICNNM), which bypasses the SVD step so as to decrease significantly the computational time. In addition, due to the extra prior knowledge encoded in the pre-learned convolution eigenvectors, ICNNM also outperforms CNNM in terms of recovery performance. Extensive experiments on video completion, prediction and frame interpolation verify the superiority of ICNNM over CNNM and several other competing methods. 
\end{abstract}


\section{Introduction}
\IEEEPARstart{M}{issing} data is a prevalent issue that compromises the reliability of the entire dataset and significantly distorts inferences drawn from it. Consequently, the task of filling in the missing entries of partially observed tensors is an important task, which is referred to as \textit{tensor completion}~\cite{4797640,5452187}. Provided that the locations of the observed entries are distributed uniformly at random (i.e., the sampling pattern is uniform), many methods with theoretical guarantees have been established in the literature~\cite{5466511,10.5555/2999611.2999705,10.5555/3157382.3157431,7536166,10.5555/3454287.3454454,8606166,WU2023646}. Apparently, the sampling pattern in realistic environment may not be uniform, attracting considerable attentions to the direction of \textit{tensor completion with nonuniform sampling}, e.g.,~\cite{norandom1,norandom2,norandom5,norandom6}, in which the sampling pattern obeys specific but possibly nonuniform distributions.\par

However, in modern applications such as data forecasting~\cite{9548664,liu2022recovery,liu2022time} and  image/video super-resolution~\cite{videoso,9048735,10465659}, the locations of observed entries may not follow any distributions of explicitly known forms. To cover such difficult cases,~\cite{liu2022recovery} suggested a rather inclusive problem termed \textit{tensor completion with arbitrary sampling} (TCAS):

\begin{definition} Let $L_0 \in \mathbb{R}^{m_1\times \cdots \times m_n}$ be the target tensor of order $n$ ($n\geq1$), and $\Omega \subset \left\{1, \cdots, m_1\right\} \times \cdots \times \left\{1, \cdots, m_n\right\}$ be the sampling set consisting of locations of observed entries. Now we know a subset of entries in $L_0$ according to $\Omega$, but their locations may be \textbf{arbitrarily} distributed. Can we recover $L_0$ in an exact fashion?\end{definition}

To address the above TCAS problem,~\cite{liu2022recovery} proposed a convex program known as Convolution Nuclear Norm Minimization (CNNM):
\begin{align}
  \label{eq:cnnm}
  &\min_{L} \left\| \mathcal{A}_k(L) \right\|_* ,\quad
  \text{s.t.}  \mathcal{P}_{\Omega}(L) = \mathcal{P}_{\Omega}(L_0),
\end{align}
where $\mathcal{A}_k$ is a linear map from $\mathbb{R}^{m_1\times \cdots \times m_n}$ to $\mathbb{R}^{m\times k}$ such that $\mathcal{A}_k(\cdot)$ generates the \textit{convolution matrix} of tensors, $\left\|\cdot \right\|_*$ is the \textit{nuclear norm}~\cite{envelope} of matrices, $\mathcal{P}_{\Omega}$ denotes the orthogonal projection onto $\Omega$, $m=\Pi_{j=1}^nm_j$, and $k$ is a parameter. By and large, CNNM is to seek a tensor that not only meets the observed entries on the sampling set but also attains the lowest \textit{convolution rank} (i.e., the rank of the convolution matrix). Due to the strengths of the convolution operator, CNNM provably solves TCAS, as long as the target $L_0$ is \textit{convolutionally low-rank} (i.e., $\mathcal{A}_k(L_0)$ is low-rank). It has also been proven that the smooth signals with bounded values, e.g., images and videos, possess the convolutional low-rankness property~\cite{liu2022recovery}. Despite its theoretical soundness and promising performance, CNNM requires performing Singular Value Decomposition (SVD) on fairly large convolution matrices several times during inference, which is very costly and hard to accelerate by parallel units such as Graphics Processing Unit (GPU).

To resolve the issue, we shall introduce in this paper a novel method termed Inductive Convolution Nuclear Norm Minimization (ICNNM). First notice that the nuclear norm of convolution matrix can be rewritten as in the following:
\begin{align}\label{eq:cnnorm}
&\left\|\mathcal{A}_k(L)\right\|_* = \sum_{i=1}^k\|L\star\kappa_i(L)\|_F,
\end{align}
where $\kappa_i(L)\in\mathbb{R}^{k_1\times\cdots\times{}k_n}$ is the $i$th \textit{convolution eigenvector} of $L$, $\star$ denotes the convolution operator, and $k_1\times\cdots\times{}k_n$ is the \textit{kernel size} used for defining the convolution matrix $\mathcal{A}_k(L)$ with $k=\Pi_{j=1}^nk_j$. In general, the so-called convolution eigenvectors play a similar role as the feature filters in Convolution Neural Networks~\cite{LeNet5:1998}, and it has been well-known that the feature filters can be shared across different data tensors. This drives us to replace in Eq~\eqref{eq:cnnorm} the variables $\{\kappa_i(L)\}_{i=1}^k$ with the convolution eigenvectors pre-learned from some reference tensors that are somehow relevant to the target $L_0$. Since the pre-learned convolution eigenvectors are fixed during inference, the SVD step is avoided and the computationally time is dramatically reduced especially when GPUs are used. Besides its superiority in computational efficiency, ICNNM can also surpass CNNM in terms of recovery performance, because the pre-learned convolution eigenvectors may contain prior knowledge helpful for identifying $L_0$. To summarize, the contributions of this paper mainly include:
\begin{itemize}
\item[$\bullet$] We reformulate the optimization objective of CNNM using pre-learned convolution eigenvectors and propose a novel method termed ICNNM for addressing the TCAS problem. Extensive experiments show that our ICNNM is better than CNNM, in terms of both computational efficiency and recovery performance.\
\item[$\bullet$] We theoretically analyze the conditions under which ICNNM can succeed in recovering the target $L_0$ from a subset of its entries sampled in an arbitrary manner. Remarkably, our analysis helps to explain why ICNNM can surpass CNNM in regard to recovery performance. 
\end{itemize}   

\begin{figure*}[htbp]
    \centering
    \captionsetup{font=small} 
    \begin{subfigure}[b]{0.16\linewidth} 
        \centering
        \includegraphics[width=\linewidth]{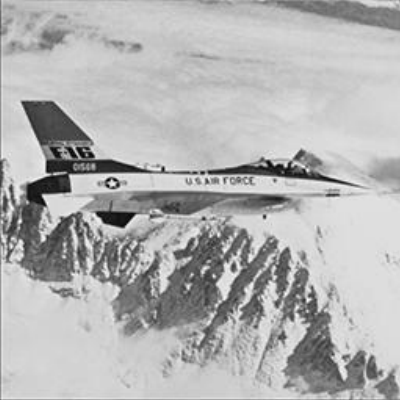}
        \vspace{0.25cm} 
        \caption{Airplane}
        \label{fig:gt}
    \end{subfigure}%
    \hspace{0.01\linewidth} 
    \begin{subfigure}[b]{0.16\linewidth}
        \centering
        \includegraphics[width=\linewidth]{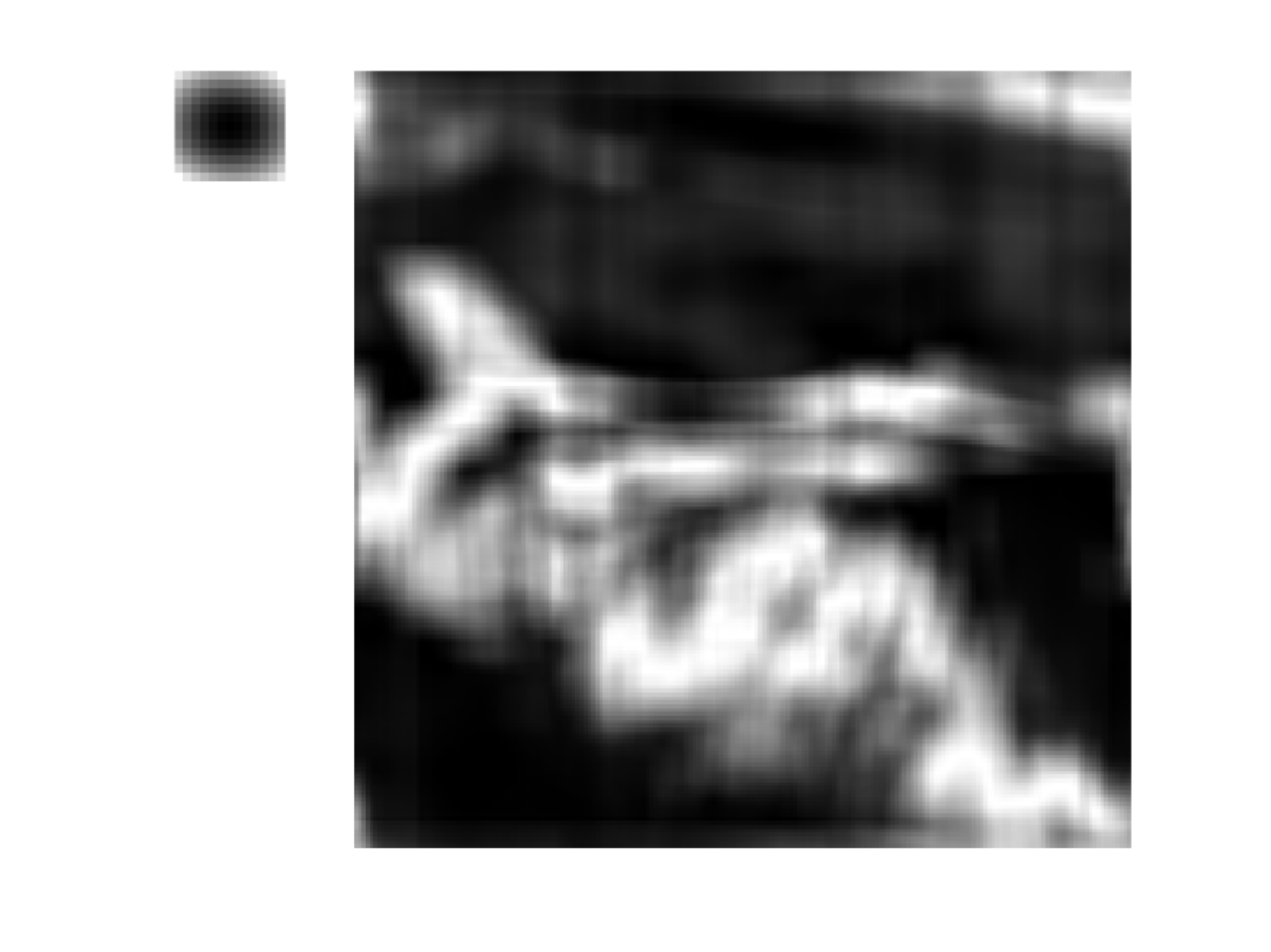} \\
        \vspace{0.1cm} 
        \includegraphics[width=\linewidth]{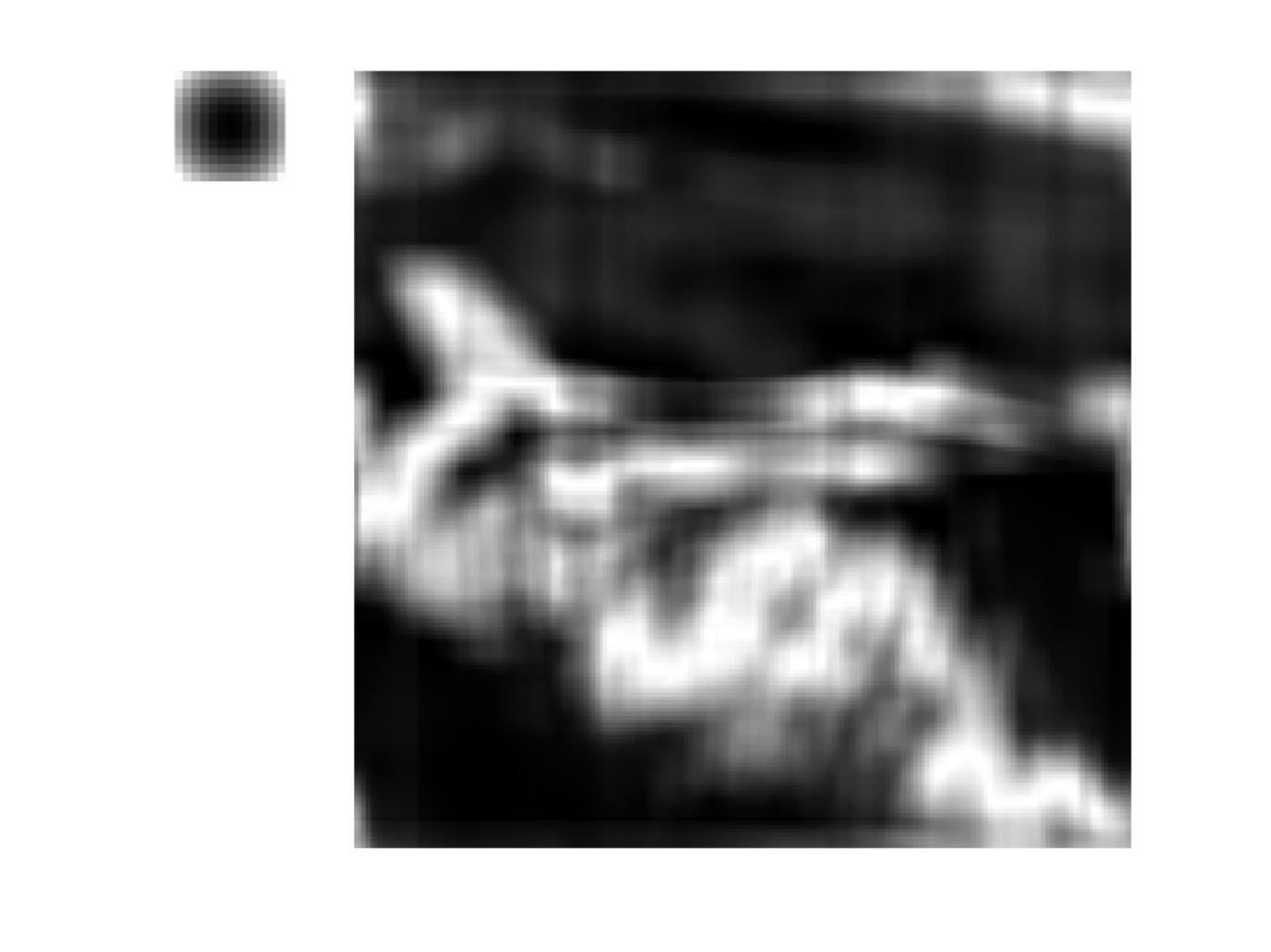}
        \caption{$i=1$}
    \end{subfigure}%
    \hspace{0.01\linewidth}
    \begin{subfigure}[b]{0.16\linewidth}
        \centering
        \includegraphics[width=\linewidth]{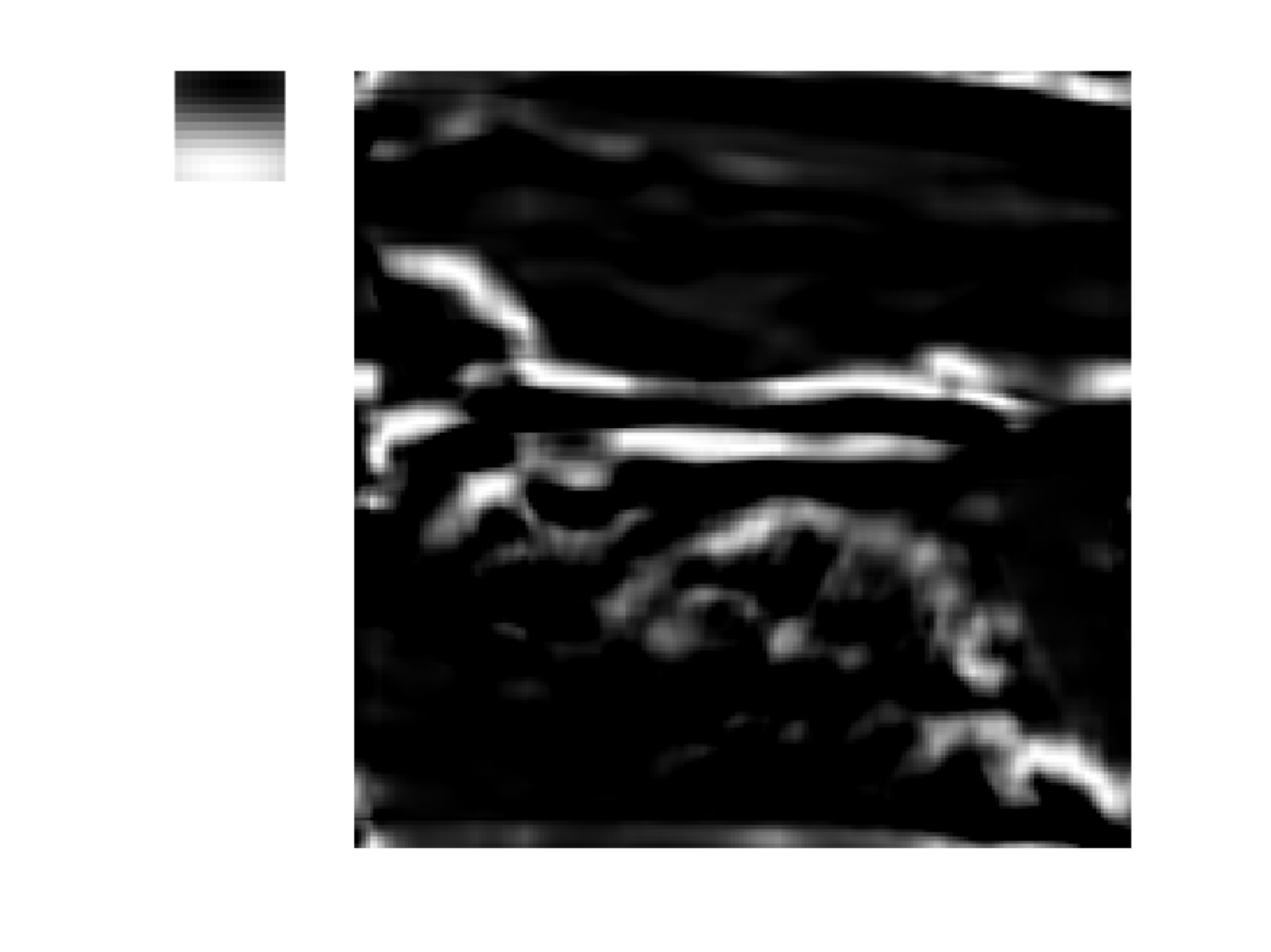} \\
        \vspace{0.1cm}
        \includegraphics[width=\linewidth]{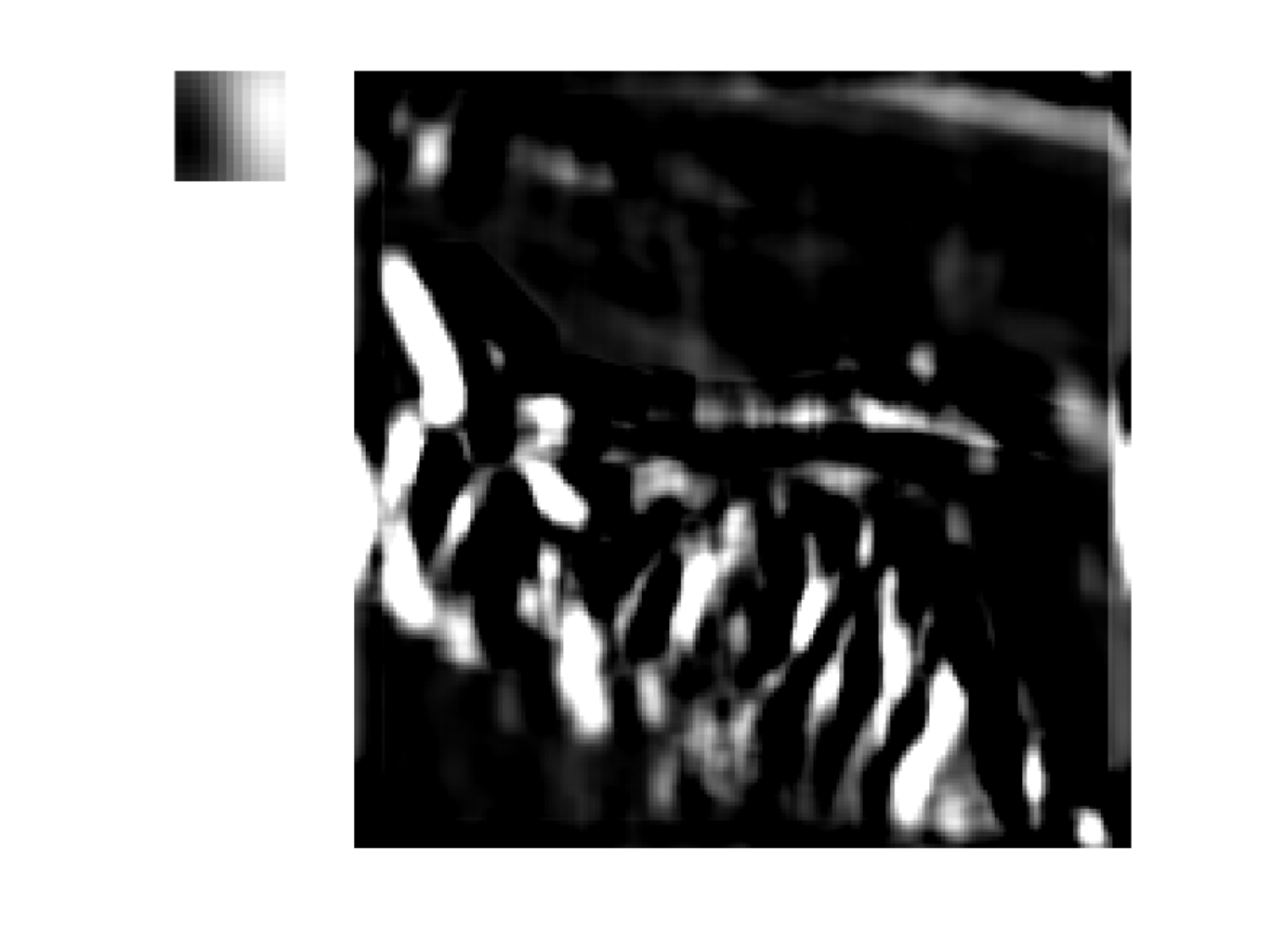}
        \caption{$i=2$}
    \end{subfigure}%
    \hspace{0.01\linewidth}
    \begin{subfigure}[b]{0.16\linewidth}
        \centering
        \includegraphics[width=\linewidth]{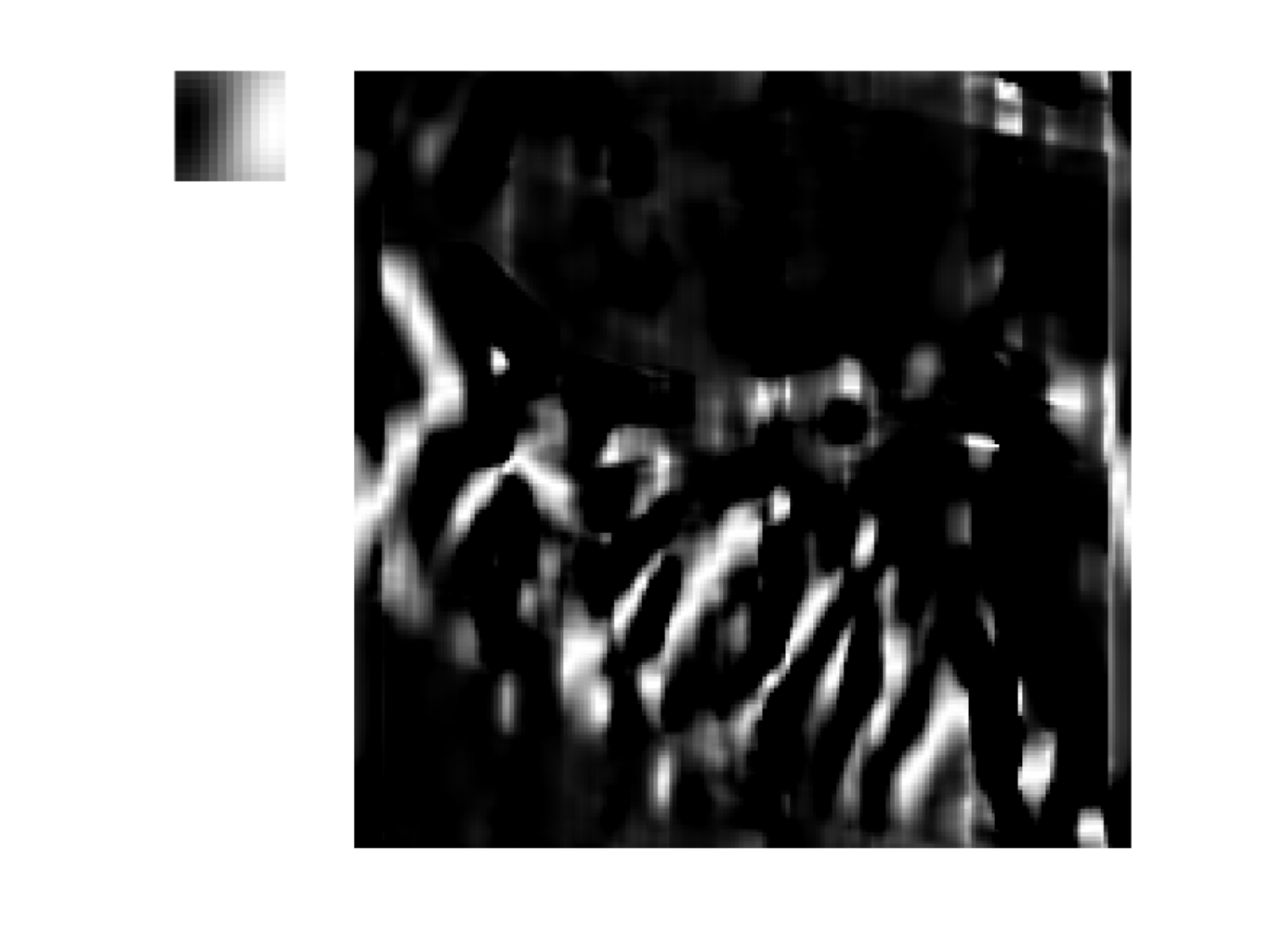} \\
        \vspace{0.1cm}
        \includegraphics[width=\linewidth]{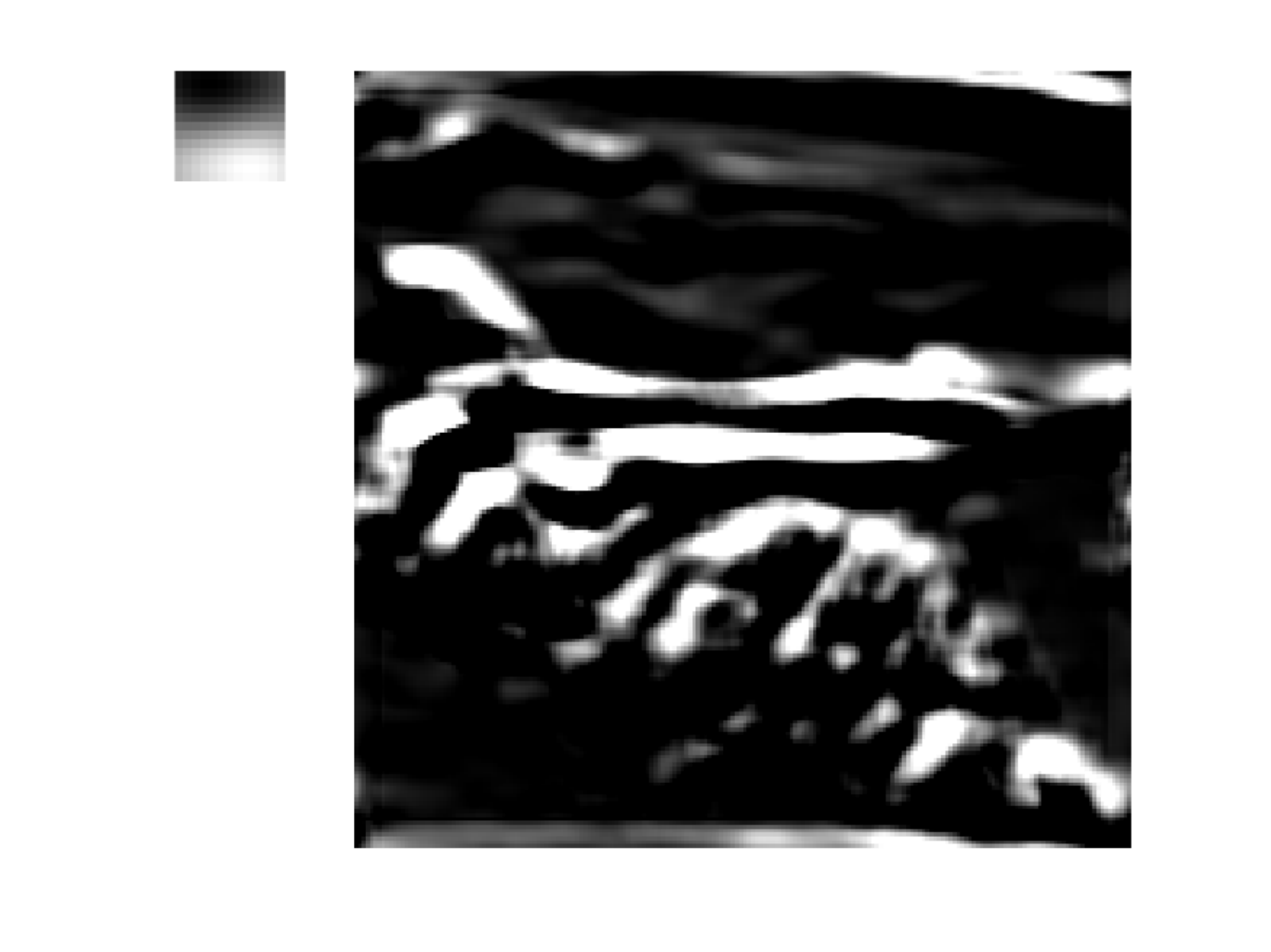}
        \caption{$i=3$}
    \end{subfigure}%
    \hspace{0.01\linewidth}
    \begin{subfigure}[b]{0.16\linewidth}
        \centering
        \includegraphics[width=\linewidth]{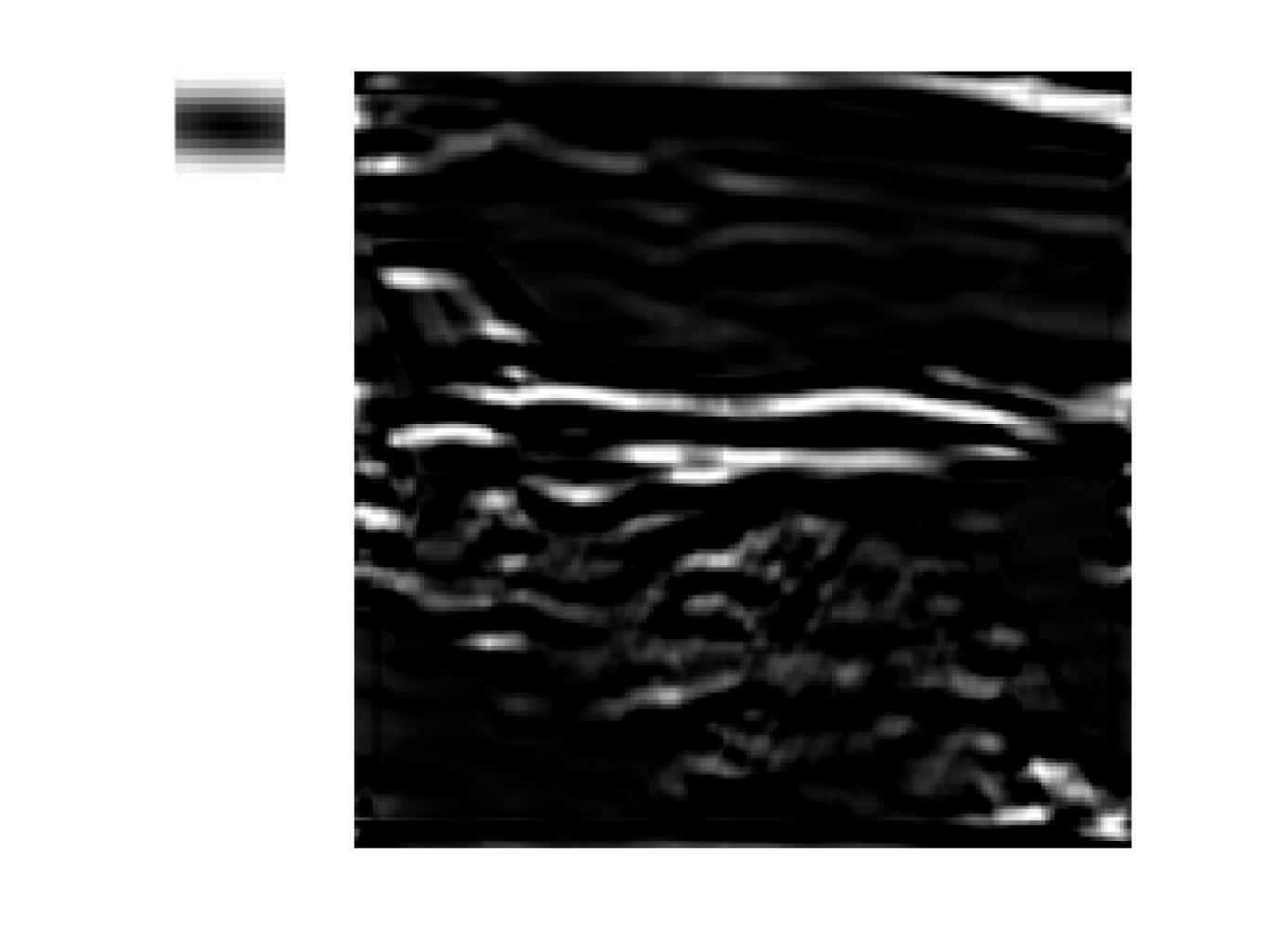} \\
        \vspace{0.1cm}
        \includegraphics[width=\linewidth]{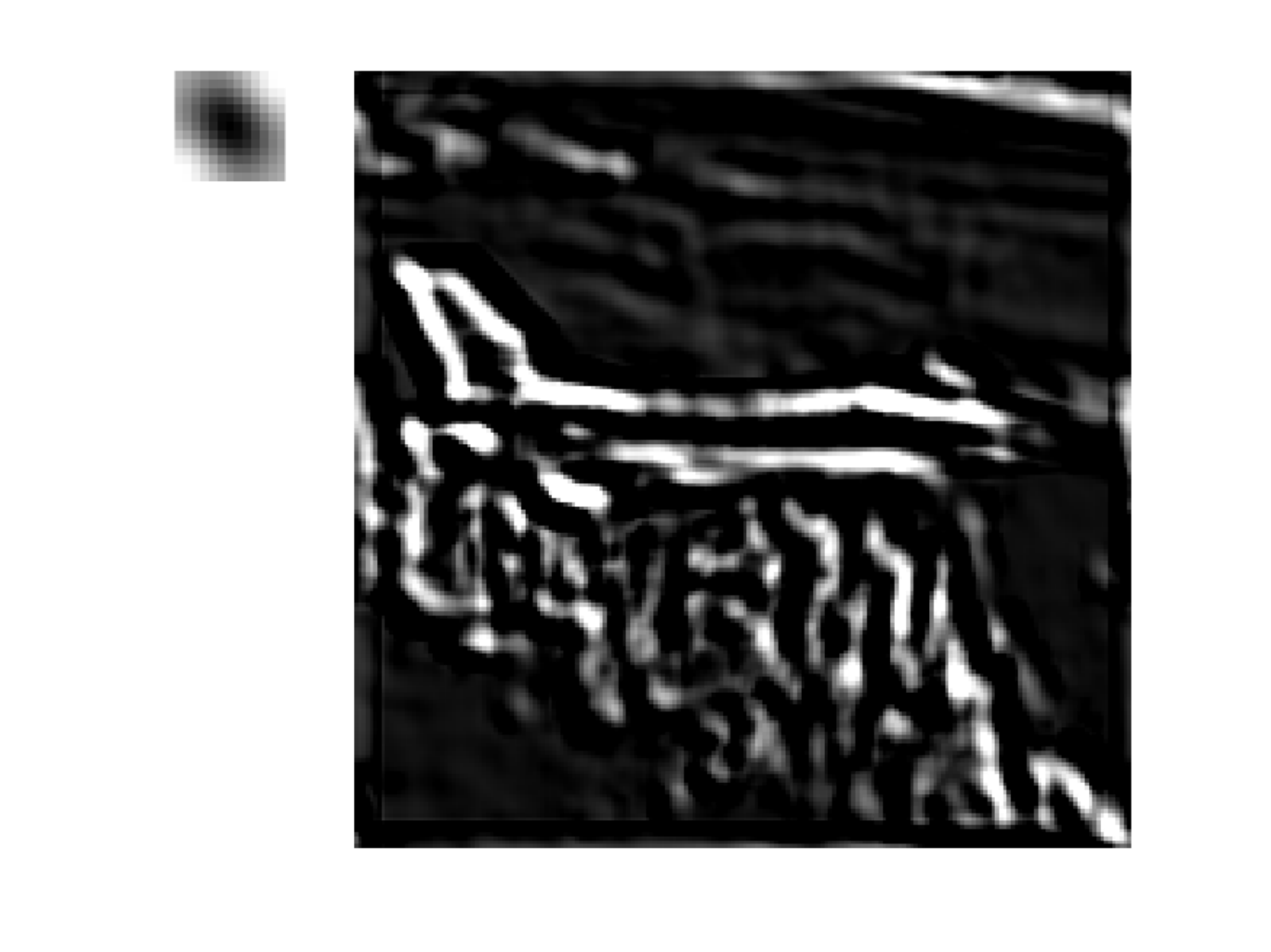}
        \caption{$i=4$}
    \end{subfigure}%
    
    \vspace{0.2cm} 
    \caption{Visualization of the first four convolution eigenvectors $\kappa_i$ and filtered signals $X \star \kappa_i$ (top row), with corresponding inductive convolution eigenvectors $\bar{\kappa}_i$ and filtered signals $X \star \bar{\kappa}_i$ (bottom row). Inductive eigenvectors were learned using seven Set8 images. All eigenvectors were normalized to [0,1] and filtered signals to max=1 for visualization. Note the similarity between image-specific and inductive eigenvectors, demonstrating transitivity across images.}
    \label{fig:Supplementary1}
\end{figure*}

\section{Notations and Preliminaries}
\subsection{Summary of Main Notations}
Capital letters, such as $L$, are used to represent tensors of order $n$ with $n\geq1$, while lowercase letters represent integers and Greek letters usually represent important parameters. It is worth noting that $m_1 \times \cdots \times m_n$ and $k_1 \times \cdots \times k_n$ are reserved to denote respectively the tensor dimension and the kernel size, with $m=\prod_{i=1}^{n}m_i$ and $k=\prod_{i=1}^{n}k_i$.\par

For an order-$n$ tensor $L$, $[L]_{i_1,\cdots{},i_n}$ is its $(i_1,\cdots,i_n)$th entry, and its Frobenius norm is given by $\|L\|_F=\sqrt{\sum_{i_1,\cdots{},i_n}([L]_{i_1,\cdots{},i_n})^2}$. Four matrix norms are used frequently: For a matrix $M$, $\|M\|_*$ represents its nuclear norm (i.e., the sum of singular values), $\|M\|_{2,1}$ is the sum of $\ell_{2}$ norms of its columns, $\|M\|_{2,\infty}$ is the maximum $\ell_{2}$ norm among its columns, and $\|M\|$ denotes its operator norm (i.e., the largest singular value).  

Specially, $\Omega \subset \left\{1,\cdots,m_1\right\} \times \cdots \times \left\{1,\cdots,m_n\right\}$ denotes the sampling set introduced earlier, and $\mathcal{P}_{\Omega}$ denotes the orthogonal projection onto $\Omega$; namely,
\begin{align*}
    [\mathcal{P}_{\Omega}(L)]_{i_1,\cdots,i_n} = \left\{\begin{array}{cc}[L]_{i_1,\cdots,i_n}, & \text{if }(i_1,\cdots,i_n)\in\Omega,\\
    0, & \text{otherwise.}\end{array}\right.
\end{align*}
The symbol $|\cdot|$ is reserved to denote the cardinality of a set, e.g., $|\Omega|$ stands for the number of elements in $\Omega$. Similarly, the symbol $\mathcal{P}_{\Omega}^\perp$ denotes the orthogonal projection onto the complement space of $\Omega$; that is, $\mathcal{P}_{\Omega} + \mathcal{P}_{\Omega}^\perp = \mathcal{I}$, where $\mathcal{I}$ is the identity operator.

\subsection{Convolution Matrix and Convolution Sampling Set}\label{sec:Convolution Matrix}
The concepts of convolution matrix and convolution sampling set are detailed in~\cite{liu2022recovery}. For the ease of reading, we shall briefly introduce them in this subsection. 

Let $L \in \mathbb{R}^{m_1\times \cdots \times m_n}$ and $X \in \mathbb{R}^{k_1\times\cdots\times k_n}$ be two tensors with $k_i \le m_i,\forall 1 \le i \le n$. Denote by $L \star X\in \mathbb{R}^{m_1\times \cdots \times m_n}$ the circular convolution of $L$ and $X$. Then the $(i_1,\cdots,i_n)$th entry of $L \star X$ is given by
\begin{equation*}
  [L \star X]_{i_1,\cdots,i_n} =\sum\limits_{s_1,\cdots,s_n}[L]_{i_1-s_1,\cdots,i_n-s_n}[X]_{s_1,\cdots,s_n},
\end{equation*}
where it is assumed that $i_j-s_j\equiv{}i_j-s_j+m_j$ for $i_j\leq{}s_j.$ Let $\mathrm{vec}(\cdot)$ be the vectorization operator that converts tensors into long vectors. Then we have
\begin{equation}
  \label{eq:ak}
  \mathrm{vec}(L \star X)=\mathcal{A}_k(L)\mathrm{vec}(X),\forall L,X,
\end{equation}
where $\mathcal{A}_k(L) \in \mathbb{R}^{m \times k}$ with $m=\prod_{i=1}^{n}m_i$ and $k=\prod_{i=1}^{n}k_i$ is the convolution matrix of the tensor $L$.

For a given sampling set $\Omega\subset\left\{1, \cdots, m_1\right\} \times \cdots \times \left\{1, \cdots, m_n\right\}$, its convolution sampling set, denoted as $\Omega_{\mathcal{A}}\subset\left\{1, \cdots, m\right\} \times \left\{1, \cdots, k\right\}$, can be actually regarded as the convolution matrix of $\Omega$. To be more precise, define the mask tensor of $\Omega$ as $\Theta_{\Omega}$, with
\begin{align*}
[\Theta_\Omega]_{i_1,\cdots{},i_n}=\left\{\begin{array}{cc}
1,&\text{if }(i_1,\cdots,i_n)\in\Omega,\\
0, &\text{otherwise.}
\end{array}\right.
\end{align*}
Then the convolution sampling set $\Omega_{\mathcal{A}}$ of $\Omega$ is given by
\begin{align}\label{eq:convset}
\Theta_{\Omega_{\mathcal{A}}} = \mathcal{A}_k(\Theta_{\Omega})  \quad \text{and} \quad \Omega_{\mathcal{A}} = \textrm{supp} (\Theta_{\Omega_{\mathcal{A}}}),
\end{align}
where $\Theta_{\Omega_{\mathcal{A}}}$ is the mask matrix of $\Omega_{\mathcal{A}}$, and $\mathrm{supp}(\cdot)$ denotes the support set of a tensor. Interestingly, the convolution sampling set has a nice property: Regardless of how the observed entries are sampled, each column of $\Theta_{\Omega_{\mathcal{A}}}$ has exactly $\rho_0m$ ones and $(1-\rho_0)m$ zeros, and each row of $\Theta_{\Omega_{\mathcal{A}}}$ has at most $(1-\rho_0)m$ zeros, where $\rho_0=|\Omega|/m$ is the sampling rate. This is indeed the primary reason of why CNNM can handle arbitrary sampling patterns, as revealed by the theories in~\cite{8815944}. 
\subsection{Convolution Eigenvectors}\label{sec:ConvEg}
The concept of convolution eigenvalues and eigenvectors was first introduced in~\cite{liu2014blind}, under the context of matrices, and was generalized to tensors of any order $n\geq1$ by~\cite{liu2022recovery}. Here, we shall briefly present them for the ease of reading. 

Let the kernel size be $k_1\times \cdots \times k_n$. Then for a tensor $L \in \mathbb{R}^{m_1\times \cdots \times m_n}$ with $m_j \geq k_j,\forall{}1\leq{j}\leq{n}$, its first convolution eigenvalue, denoted as $\sigma_1(L)$, is given by
\begin{align*}
\sigma_1(L)&=\max_{X\in\mathbb{R}^{k_1\times\cdots\times{}k_n}}\|L\star{}X\|_F,\textrm{ s.t. } \|X\|_F=1.
\end{align*}
The optimal solution to the above problem is called the first convolution eigenvector, denoted as $\kappa_1(L)\in\mathbb{R}^{k_1\times\cdots\times{}k_n}$. Similarly, the $i$th ($2\le i\le k$) convolution eigenvalue of $L$ is defined as
\begin{eqnarray*}&\sigma_i(L)=\max_{X\in\mathbb{R}^{k_1\times\cdots\times{}k_n}}\|L\star{}X\|_F, \\\nonumber
 &\textrm{s.t.} \|X\|_F=1, \langle{}X,\kappa_j(L)\rangle=0,\forall{}j<i,
\end{eqnarray*}
where $\langle\cdot,\cdot\rangle$ denotes the inner product between two tensors. The optimal solution to above problem is the $i$th convolution eigenvector, denoted as $\kappa_i(L)\in\mathbb{R}^{k_1\times\cdots\times{}k_n}$.

According to~\eqref{eq:ak}, the convolution eigenvalues and eigenvectors of a tensor $L$ exactly correspond to the singular values and right singular vectors of the convolution matrix $\mathcal{A}_k(L)$, respectively. Hence, the nuclear norm of the convolution matrix can be reformulated as in~\eqref{eq:cnnorm}. In addition, the number of nonzero convolution eigenvalues of a tensor, called the \emph{convolution rank}, is simply the rank of the convolution matrix of the tensor.

\subsection{Convolution Coherence}
As explained in~\cite{liu2022recovery}, the so-called \textit{convolution coherence} of a tensor $L$ is exactly the standard \emph{coherence}~\cite{4797640} of the convolution matrix $\mathcal{A}_k(L)$. Let the skinny SVD of a rank-$r$ matrix $B\in\mathbb{R}^{p\times{}q}$ be $B=U\Sigma{}V^T$, where $\Sigma\in\mathbb{R}^{r\times{}r}$. Then there are two coherence parameters, $\mu_1$ and $\mu_2$, for characterizing $B$:
\begin{align}\label{eq:coherence}
&\mu_1(B) = \frac{p}{r}\max_{1\leq{}i\leq{}p}\|[U]_{i,:}\|_F^2,\\\nonumber
&\mu_2(B) = \frac{q}{r}\max_{1\leq{}i\leq{}q}\|[V]_{i,:}\|_F^2,
\end{align}
where $[\cdot]_{i,:}$ is the $i$th row of a matrix. It can be verified that $1\leq\mu_1(M)\leq{}p$ and $1\leq\mu_2(M)\leq{}q$. 

With the above notations, the two convolution coherence parameters, denoted respectively as $\mu_k^{(1)}$ and $\mu_k^{(2)}$, of a tensor $L\in\mathbb{R}^{m_1\times\cdots\times{}m_n}$ are given by
\begin{align}\label{eq:ccoherence}
\mu_k^{(1)}(L) =\mu_1(\mathcal{A}_k(L))\text{ and } \mu_k^{(2)}(L) =\mu_2(\mathcal{A}_k(L)).
\end{align}
As proven in~\cite{liu2022recovery}, the range of convolution coherence can be determined as
\begin{align}\label{eq:ccoherence:range}
1\leq\mu_k^{(1)}(L)\leq\frac{m}{k}\tau^2\text{ and } 1\leq\mu_k^{(2)}(L)\leq\tau^2,
\end{align}
where $\tau$ is the condition number of $\mathcal{A}_k(L)\in\mathbb{R}^{m\times{k}}$.

\section{Methods and Analysis}\label{sec:methods}
Since the convolution eigenvectors, which are essentially some kind of feature filters, can be shared across different tensors, it would be straightforward to replace in~\eqref{eq:cnnm} the unknown convolution eigenvectors of the optimization variable with some pre-learned convolution eigenvectors, resulting in a new convex program termed ICNNM:
\begin{equation}
  \label{eq:InductiveCNNM}
  \min_L \left\| \mathcal{A}_k(L)K \right\|_{2,1}, \quad \text{s.t.} \quad \mathcal{P}_{\Omega}(L) = \mathcal{P}_{\Omega}(L_0),
\end{equation}
where $K=[\mathrm{vec}(\bar{\kappa}_1),\cdots,\mathrm{vec}(\bar{\kappa}_k)]\in\mathbb{R}^{k\times{k}}$ is orthogonal, and $\{\bar{\kappa}_i\in\mathbb{R}^{k_1\times\cdots\times{}k_n}\}_{i=1}^k$ is a set of $k$ convolution eigenvectors learned from some reference tensors in advance. 

In the rest of this section, we shall first analyze ICNNM theoretically, then present methods for learning the convolution eigenvectors $\{\bar{\kappa}_i\in\mathbb{R}^{k_1\times\cdots\times{}k_n}\}_{i=1}^k$ as well as algorithms for solving the raised optimization problem.
\subsection{Recovery Conditions}\label{sec:conditions}
Whenever a set of pre-learned convolution eigenvectors $\{\bar{\kappa}_i\}_{i=1}^k$ are used to supplant the ``originals'', naturally, the recovery performance of ICNNM will depend on how $\{\bar{\kappa}_i\}_{i=1}^k$ is related to the target tensor $L_0$. In fact, their relationship is mainly reflected by two quantities, \textit{relative convolution rank} and \textit{spectral correlation coefficient}, which are defined as follows. 
\begin{definition}
Let $L_0\in\mathbb{R}^{m_1\times\cdots\times{}m_n}$ be a tensor of order $n$, and let $\{\bar{\kappa}_i\in\mathbb{R}^{k_1\times\cdots\times{}k_n}\}_{i=1}^k$ be a set of convolution eigenvectors, where $K=[\mathrm{vec}(\bar{\kappa}_1),\cdots,\mathrm{vec}(\bar{\kappa}_k)]$ is orthogonal. Denote by $\mathcal{I}$ the column support of $\mathcal{A}_k(L_0)K$. Then the \textit{relative convolution eigenvalues} of $L_0$ with respect to $K$ are defined as 
\begin{align*}
\sigma^K_i(L_0) = \|L_0\star\bar{\kappa}_i\|_F, \forall{}1\leq{i}\leq{k}.
\end{align*}
Accordingly, the \textit{relative convolution rank}, denoted by $r_K(\cdot)$,  is defined as the number of nonzero relative convolution eigenvalues. That is,
\begin{align}\label{eq:rkL0}
r_K(L_0) = \|\mathcal{A}_k(L_0)K\|_{2,0} =|\mathcal{I}|,
\end{align}
where $\|\cdot\|_{2,0}$ is the $\ell_{2,0}$ quasi norm that counts the number of nonzero columns of a matrix. And the \textit{spectral correlation coefficient} of $L_0$ with respect to $K$ is defined as
\begin{align}\label{eq:alphaK}
  \alpha_K(L_0)=\|\mathcal{A}_k(L_0)KD\|,
\end{align}
where $\|\cdot\|$ is the operator norm, and $D\in\mathbb{R}^{k\times{}k}$ is a diagonal matrix with the diagonal entries being
\begin{align*}
  [D]_{i,i}=\left\{\begin{array}{cc}
\frac{1}{\|L_0\star\bar{\kappa}_i\|_F},  & \textrm{if } i \in \mathcal{I}, \\
0,   & \textrm{otherwise.}
\end{array}\right.
\end{align*}
\end{definition}

In addition, the recovery performance of ICNNM also depends on the first coherence parameter of $K_{\mathcal{I}}\in\mathbb{R}^{k\times{r_K(L_0)}}$, where $K_{\mathcal{I}}$ is the submatrix of $K$ obtained by selecting the columns with indices $\mathcal{I}$. For convenience, $K_{\mathcal{I}}$ is referred to as the \textit{active submatrix} of $K$, and its coherence is specially denoted as $\mu_{K}^{\mathcal{I}}$, i.e., $\mu_{K}^{\mathcal{I}} =\mu_1(K_{\mathcal{I}})$. 

For our ICNNM to succeed in solving TCAS, as will be shown later, $\mu_{K}^{\mathcal{I}}r_K(L_0)$ and $\alpha_K(L_0)$ cannot be too large. Through some simple manipulations, it can be proven that $1\leq\alpha_K(L_0)\leq{}\sqrt{r_K(L_0)}$ and $r_K(L_0)\geq{}r_0$, where $r_0$ is the convolution rank of the target $L_0$. In the ideal case when the pre-learned convolution eigenvectors strictly coincide with the convolution eigenvectors of $L_0$, we have  $r_K(L_0)=r_0$, $\alpha_K(L_0)=1$ and $\mu_{K}^{\mathcal{I}}=\mu_k^{(2)}(L_0)$, where $\mu_k^{(2)}(L_0)$ is the second convolution coherence of $L_0$. 
\subsection{Main Results}
The following theorem guarantees that the ICNNM program~\eqref{eq:InductiveCNNM} can exactly recover the target $L_0$, as long as the sampling rate $\rho_0=|\Omega|/m$ exceeds certain threshold. 
\begin{theorem}[Noiseless]~\label{thm:noiseless} Let $L_0 \in \mathbb{R}^{m_1 \times \cdots \times m_n}$ be the target tensor and $\Omega \subset \left\{1,\cdots,m_1\right\} \times \cdots \times \left\{1,\cdots,m_n\right\}$ be the sampling set. Suppose that $\{\bar{\kappa}_i\in\mathbb{R}^{k_1\times\cdots\times{}k_n}\}_{i=1}^k$ is a given set of convolution eigenvectors and $K=[\mathrm{vec}(\bar{\kappa}_1),\cdots,\mathrm{vec}(\bar{\kappa}_k)]$
is orthogonal. Denote by $r_K(L_0)$, $\alpha_K(L_0)$ and $\mu_{K}^{\mathcal{I}}$ the relative convolution rank of $L_0$ with respect to $K$, the spectral correlation coefficient of $L_0$ with respect to $K$ and the coherence of the active submatrix of $K$, respectively. Then $L=L_0$ is the unique minimizer to the ICNNM problem~\eqref{eq:InductiveCNNM}, as long as
\begin{equation*}
  \rho_0 > 1-\frac{k}{(1+ \alpha_K^2(L_0))\mu_{K}^{\mathcal{I}}r_K(L_0)m},
\end{equation*}
where $m=\Pi_{i=1}^nm_i$ and $k=\Pi_{i=1}^nk_i$.
\end{theorem}
In the perfect case where the pre-learned convolution eigenvectors $\{\bar{\kappa}_i\}_{i=1}^k$ exactly equal to the convolution eigenvectors of the target $L_0$, the above theorem implies that the ICNNM program~\eqref{eq:InductiveCNNM} is strictly successful provided that 
\begin{align}\label{eq:lowerbound1}
\rho_0>1 - \frac{0.5k}{\mu_k^{(2)}(L_0)r_0m},
\end{align}
where $r_0$ is the convolution rank of $L_0$. In comparison, according to~\cite{liu2022recovery}, the success of the CNNM program~\eqref{eq:cnnm} in identifying $L_0$ requires 
\begin{align}\label{eq:lowerbound2}
\rho_0>1-\frac{0.25k}{\max\left(\mu_k^{(1)}(L_0),\mu_k^{(2)}(L_0)\right)r_0m}, 
\end{align}
the lower bound of which is clearly higher than \eqref{eq:lowerbound1}. Hence, our ICNNM owns potential to outperform CNNM, in the sense of recovery performance. 

In practice, the observations are often contaminated by noises, or $\mathcal{A}_k(L_0)K$ is not column-wisely sparse as equal. In this case, one should relax the equality constraint in~\eqref{eq:InductiveCNNM} and consider instead the following:
\begin{align}\label{eq:icnnm:noisy}
\min_{L} \|\mathcal{A}_k(L)K\|_{2,1},\textrm{ s.t. }\|\mathcal{P}_{\Omega}(L - M)\|_F\leq{}\epsilon,
\end{align}
where $\mathcal{P}_{\Omega}(M)$ is an observed version of $\mathcal{P}_{\Omega}(L_0)$, and $\epsilon>0$ is a parameter. In this case, we have the following theorem to guarantee the recovery performance of ICNNM.
\begin{theorem}[Noisy]\label{thm:noisy}

Adopt the same notations as in Theorem~\ref{thm:noiseless}. Suppose that $\|\mathcal{P}_{\Omega}(M - L_0)\|_F\leq\epsilon$. If
\begin{align}
\rho_0 > 1-\frac{0.64k}{(0.64+\alpha_K^2(L_0))\mu_{K}^{\mathcal{I}}r_K(L_0)m},
\end{align}
then any optimal solution $\bar{L}$ to the ICNNM problem~\eqref{eq:icnnm:noisy} approximately recovers the target $L_0$, in a sense that
\begin{align*}
&\Vert \bar{L}-L_0\Vert_{F} \le(18\sqrt{k}+2)\frac{\alpha_K(L_0)+\sqrt{1+\alpha_K^2(L_0)}}{\alpha_K(L_0)}\epsilon
\end{align*}
\end{theorem}

\subsection{Learning Convolution Eigenvectors}\label{sec:matrixk}
We shall show how to calculate the convolution eigenvectors $\{\bar{\kappa}_i\}_{i=1}^k$ used in ICNNM from a collection of $b$ reference tensors $\{L_j^{\mathrm{ref}}\}_{j=1}^b$, where $L_j\in\mathbb{R}^{m_1\times\cdots\times{}m_n}$ has the same dimension as the target $L_0$. To do this, we need to extend the concept of convolution eigenvectors defined on a single tensor to a \textit{tensor ensemble} (i.e., a collection of tensors). 
\begin{definition}\label{def:eigen} Suppose that the kernel size for order-$n$ tensors is set as $k_1\times\cdots\times{}k_n$. Then for a tensor ensemble $\{L_j^{\mathrm{ref}}\}_{j=1}^b$ with $L_j\in\mathbb{R}^{m_1\times\cdots\times{}m_n}$ $(m_t\geq{}k_t,\forall{}t)$, its first convolution eigenvalue $\sigma_1(\{L_j^{\mathrm{ref}}\}_{j=1}^b)$ is defined as
\begin{eqnarray*}
\sigma_1(\{L_j^{\mathrm{ref}}\}_{j=1}^b)=\max_{X}\sum_{j=1}^b\|L_j^{\mathrm{ref}}\star{}X\|_F^2,\textrm{ s.t. } \|X\|_F=1.
\end{eqnarray*}
The maximizer to above problem is called the first convolution eigenvector, denoted as $\kappa_1(\{L_j^{\mathrm{ref}}\}_{j=1}^b)\in\mathbb{R}^{k_1\times\cdots\times{}k_n}$. Similarly, the $i$th ($i=2,\cdots,k$, $k=\Pi_{j=1}^nk_j$) convolution eigenvalue is given by
\begin{align*}
&\sigma_i(\{L_j^{\mathrm{ref}}\}_{j=1}^b)=\max_{X}\sum_{j=1}^b\|L_j^{\mathrm{ref}}\star{}X\|_F^2, \\\nonumber
 &\textrm{ s.t. } \|X\|_F=1, \langle{}X,\kappa_t(\{L_j^{\mathrm{ref}}\}_{j=1}^b)\rangle=0,\forall{}t<i.
\end{align*}
The maximizer to above problem is the $i$th convolution eigenvector, denoted as $\kappa_i(\{L_j^{\mathrm{ref}}\}_{j=1}^b)\in\mathbb{R}^{k_1\times\cdots\times{}k_n}$.
\end{definition}
Subsequently, we use the convolution eigenvectors of the tensor ensemble $\{L_j^{\mathrm{ref}}\}_{j=1}^b$ to construct the orthogonal matrix $K$ required by ICNNM. Here, the computation is done by finding the right singular vectors of the following matrix:
\begin{equation}
  \label{eq:A}
    A= \left [ \begin{matrix}\mathcal{A}_k(L_1^{\mathrm{ref}}) \\ \vdots \\ \mathcal{A}_k(L_b^{\mathrm{ref}}) \\ \end{matrix} \right ]\in\mathbb{R}^{mb\times{}k}. 
\end{equation}
The computational cost of the above pre-learning procedure is dominated by the calculation of $A^TA$, which has a complexity of $\mathcal{O}\left(mbk^2\right)$. 

\subsection{Optimization Algorithm}
In practice, program~\eqref{eq:InductiveCNNM} is seldom used, and thus we only show how to solve the optimization problem~\eqref{eq:icnnm:noisy}. For the ease of implementation, we reformulate equivalently the problem~\eqref{eq:icnnm:noisy} as follows:
\begin{equation}
  \label{eq:admm}
  \min_L \left\| \mathcal{A}_k(L)K \right\|_{2,1} +\frac{\lambda{}k}{2}\|\mathcal{P}_{\Omega}(L-M)\|_F^2,
\end{equation}
where $\lambda>0$ is a hyper-parameter ($\lambda=1000$ in all our experiments), and the amplification by a factor of $k=\Pi_{j=1}^nk_j$ is for the purpose of rescaling the two terms to a similar level. Subsequently, the above problem is solved by leveraging Alternating Direction Method of Multipliers (ADMM)~\cite{admmconv1,lin2010augmented} to solve the following equivalent problem:
\begin{align*}
  \min_{L,Z} \left\| Z \right\|_{2,1} +\frac{\lambda{}k}{2}\| \mathcal{P}_{\Omega}(L-M)\|_F^2,\text{ s.t. } Z = \mathcal{A}_k(L)K, 
\end{align*}
where $Z$ is an auxiliary variable. The corresponding augmented Lagrangian function is
\begin{align*}
\mathcal{L}(L, Z, Y, \theta)&=\left\| Z \right\|_{2,1} +\frac{\lambda{}k}{2}\| \mathcal{P}_{\Omega}(L-M)\|_F^2 \\&+\langle{}Z- \mathcal{A}_k(L)K,Y\rangle+\frac{\theta}{2}\| Z- \mathcal{A}_k(L)K\|_F^2,
\end{align*}
where $Y$ is the Lagrange multiplier and $\theta>0$ is the penalty parameter. Then the above augmented Lagrangian function is minimized alternately. Namely, while fixing the other variables, the variable $Z$ is updated by
\begin{equation*}
  Z=\arg\min_{Z} \frac{1}{\theta}\left\| Z \right\|_{2,1} +\frac{1}{2}\| Z- (\mathcal{A}_k(L)K-Y/\theta)\|_F^2,
\end{equation*}
which has a closed-form solution given by Lemma 3.2 of~\cite{21normcompute}. While keeping the other variables fixed, the variable $L$ is updated as follows:
\begin{equation*}
  L =\frac{\mathcal{A}_k^{*}((Y+\theta Z)K^T)/k+\lambda \mathcal{P}_{\Omega}(M)}{\lambda \Theta_\Omega +\theta},
\end{equation*}
where $\Theta_\Omega$ stands for the mask tensor of $\Omega$, $\mathcal{A}_k^*$ is a map from 
$\mathbb{R}^{m\times{k}}$ to $\mathbb{R}^{m_1\times\cdots\times{}m_n}$ and called the \textit{Hermitian adjoint} of $\mathcal{A}_k$~\cite{liu2022recovery}, and the division on tensors is simply the element-wise tensor division. Finally, the Lagrange multiplier $Y$ and the penalty $\theta>0$ are updated according to the instructions of ADMM~\cite{admmconv1,lin2010augmented}.

In general, the computational cost of the above optimization procedure is dominated by calculating the multiplication of $m\times{}k$ and $k\times{}k$ matrices, which has a complexity of $\mathcal{O}(mk^2)$. By contrast, the most expensive step of CNNM is to perform SVD on $m\times{}k$ matrices, which also has $\mathcal{O}(mk^2)$ complexity. So, from the viewpoint of computational complexity, there seems no difference between ICNNM and CNNM. Nevertheless, even taking the pre-learning procedure into account, ICNNM can be many times faster than CNNM, as the computation of matrix multiplication is much cheaper and more parallelizable than SVD.

\section{Mathematical Proof}
In this section, we present the detailed proofs of the proposed theorems.
\subsection{Proof to Theorem \ref{thm:noiseless}}
The proof follows a standard line of analysis. We begin by establishing the following lemma, which specifies the conditions under which the solution to~\eqref{eq:InductiveCNNM} is unique and exact.
 \begin{lemma}[Noiseless]~\label{lemma:noiseless} Use the same notations as in
Theorem \ref{thm:noiseless} and let $\mathcal P_{K_{\mathcal I}}(\cdot) :=(\cdot) K_{\mathcal I} K_{\mathcal I}^T$ with $\mathcal{P}_{K_\mathcal{I}}^\perp(\cdot)=\mathcal{I}-P_{K_{\mathcal I}}(\cdot)$. The $L=L_0$ is the unique minimizer to the ICNNM problem~\eqref{eq:InductiveCNNM}, as long as: 

\begin{enumerate}[]
\item $\mathcal{P}_{\Omega_A}^\perp \bigcap \mathcal{P}_{K_\mathcal{I}} ={0}$
\item  $\exists Y \in R^{m \times k},\mathcal{P}_{K_\mathcal{I}}\mathcal{P}_{\Omega_A}(Y)=\mathcal{A}_k(L_0)K_\mathcal{I}DK_\mathcal{I}^T,\\ ||\mathcal{P}_{K_\mathcal{I}}^\perp \mathcal{P}_{\Omega_A}(Y)K||_{2,\infty}<1$
\end{enumerate}
\end{lemma}

\textit{Proof:}

 Denote $F_0=\mathcal{P}_{K_\mathcal{I}}^\perp \mathcal{P}_{\Omega_A}(Y)K$, then we have
\begin{equation}
\label{eq:proof1}
\begin{aligned}
&\mathcal{A}_k^*(\mathcal{A}_k(L_0)KDK^T+F_0K^T)\\
&=\mathcal{A}_k^*(\mathcal{P}_{\Omega_A}(Y)K^T)\\ 
&=\mathcal{P}_{\Omega}(\mathcal{A}_k^*(YK^T))\\ 
&\in \mathcal{P}_{\Omega}
\end{aligned}
\end{equation}
\par
Problem (\ref{eq:InductiveCNNM}) is convex, meaning that any feasible point satisfying a valid subgradient optimality condition is a global minimizer. Since the constructed Y certifies that meets this condition, $L_0$ is a global optimal solution to the problem in (\ref{eq:InductiveCNNM}). A proof of the uniqueness of $L_0$ then follows by showing, via subgradient bounds, that any other feasible perturbation must strictly increase the objective.\par
According to Lemma 2.2 in~\cite{7152948}, for the function $f(L)=||\mathcal{A}_k(L)K||_{2,1}$, if there exists a matrix $F$ such that $\mathcal{P}_{K_\mathcal{I}}(F)=0$ and $||F||_{2,\infty} \le 1$, then $\mathcal{A}^*_k(\mathcal{A}_k(L)KDK^T+FK^T)$ is the subgradient of $f(L)$.\par
Consider another feasible solution $L_0+\Delta$ with $\Delta \in P_\Omega^\perp$. We need to demonstrate that the value of the optimization objective increases unless $\Delta=0$. \par
Due to the duality between the $||\cdot||_{2,1}$ and $||\cdot||_{2,\infty}$ norms, and given the conditions specified in the previous paragraph, we can always find a matrix $F$ such that
\begin{equation}
\label{eq:prooff}
\left<F,\mathcal{A}_k(\Delta)K\right>=||\mathcal{P}_{K_\mathcal{I}}^\perp (\mathcal{A}_k(\Delta)K)||_{2,1}
\end{equation}
Thus, we have
\begin{equation}
\label{eq:proof2}
\begin{aligned}
& ||\mathcal{A}_k(L_0 + \Delta)K||_{2,1} - ||\mathcal{A}_k(L_0)K||_{2,1} \\
\geq & \left< \mathcal{A}_k(L_0)K_\mathcal{I}D + F, \mathcal{A}_k(\Delta)K \right> \\
= & \left< \mathcal{P}_{\Omega_A}(Y) + F - F_0, \mathcal{A}_k(\Delta)K \right> \\
\geq & (1 - ||F_0||_{2,\infty}) || \mathcal{P}_{K_\mathcal{I}}^\perp \mathcal{A}_k(\Delta)K ||_{2,1} \\
& + \left< \mathcal{P}_{\Omega_A}(Y), \mathcal{A}_k(\Delta)K \right>
\end{aligned}
\end{equation}

According to Lemma 2.1 in ~\cite{liu2022recovery}, the second term in equation (\ref{eq:proof2}) is exactly zero. Thus, we have
\begin{equation}
\label{eq:prooffinal}
\begin{aligned}
&||\mathcal{A}_k(L_0+\Delta)K||_{2,1}-||\mathcal{A}_k(L_0)K||_{2,1}\\&\geq (1-||F_0||_{2,\infty})||\mathcal{P}_{K_\mathcal{I}}^\perp \mathcal{A}_k(\Delta)K||_{2,1}
\end{aligned}
\end{equation}
Given that $||F_0||_{2,\infty} \le 1$ and $\mathcal{P}_{\Omega_A}^\perp \bigcap \mathcal{P}_{K_\mathcal{I}} ={0}$,$||\mathcal{A}_k(L_0+\Delta)K||_{2,1}$ is greater than $||\mathcal{A}_k(L_0)K||_{2,1}$ unless $\Delta=0$ considering that $\mathcal{A}_k(\Delta) \in \mathcal{P}_{\Omega_A}^\perp$.$\hfill\blacksquare$\par
Having established that Lemma \ref{lemma:noiseless} ensures exact recovery, we next identify the conditions required for the Lemma \ref{lemma:noiseless} to hold and show how they lead directly to Theorem \ref{thm:noiseless}.
\par
Lemma 5.6 and 5.8 in~\cite{8815944} imply that $K_\mathcal{I}$ is $\Omega_A^T$-isomeric if and only if $\mathcal{P}_{\Omega_A}^\perp \bigcap \mathcal{P}_{K_\mathcal{I}} =\{0\}$ under which condition $\mathcal{P}_{K_\mathcal{I}}\mathcal{P}_{\Omega_A}\mathcal{P}_{K_\mathcal{I}}$ is invertible, define $Y$ as
\begin{equation}
\label{eq:prooffinaly}
Y= \mathcal{P}_{\Omega_A}\mathcal{P}_{K_\mathcal{I}}(\mathcal{P}_{K_\mathcal{I}}\mathcal{P}_{\Omega_A}\mathcal{P}_{K_\mathcal{I}})^{-1}(\mathcal{A}_k(L_0)K_\mathcal{I}DK_\mathcal{I}^T)
\end{equation}
It follows directly that $\mathcal{P}_{K_\mathcal{I}}\mathcal{P}_{\Omega_A}(Y)=\mathcal{A}_k(L_0)K_\mathcal{I}DK_\mathcal{I}^T$. Then

\begin{equation}
\label{eq:proofalphanorm}
\begin{aligned}
& \Vert \mathcal{P}_{K_\mathcal{I}}^\perp \mathcal{P}_{\Omega_A}(Y)K \Vert_{2,\infty}  \le \Vert \mathcal{P}_{K_\mathcal{I}}^\perp \mathcal{P}_{\Omega_A}(Y) \Vert \Vert K \Vert_{2,\infty} \\
& \le \Vert \mathcal{P}_{K_\mathcal{I}}^\perp \mathcal{P}_{\Omega_A} \mathcal{P}_{K_\mathcal{I}} (\mathcal{P}_{K_\mathcal{I}} \mathcal{P}_{\Omega_A} \mathcal{P}_{K_\mathcal{I}})^{-1} \Vert \Vert \mathcal{A}_k(L_0) K_\mathcal{I} D K_\mathcal{I}^T \Vert \\
& \le \alpha_K(L_0) \cdot \Vert \mathcal{P}_{K_\mathcal{I}}^\perp \mathcal{P}_{\Omega_A} \mathcal{P}_{K_\mathcal{I}} (\mathcal{P}_{K_\mathcal{I}} \mathcal{P}_{\Omega_A} \mathcal{P}_{K_\mathcal{I}})^{-1} \Vert
\end{aligned}
\end{equation}

Denote by $\rho$ the smallest fraction of observed entries in each row and column of $A_k(L_0)$. Provided that $\rho_0 > 1-\frac{k}{(1+ \alpha_K^2(L_0))\mu_{K}^{\mathcal{I}}r_K(L_0)m}$, we have
\begin{equation}
\begin{aligned}
\rho \textgreater 1-\frac{1}{(1+ \alpha_K^2(L_0))\mu_{K}^{\mathcal{I}}r_K(L_0)} \\
\end{aligned}
\end{equation}
Following from Theorem 3.4 in~\cite{8815944} with $K_\mathcal{I}$ being $\Omega_A^T$-isomeric, we have $\gamma_{\Omega_A^T} (K_\mathcal{I}^T) \textgreater \frac{\alpha_K^2(L_0)}{1+ \alpha_K^2(L_0)}$. And with the help of Lemma 5.10 in~\cite{8815944}, we have that:

\begin{equation}
\begin{aligned}
&\Vert \mathcal{P}_{K_\mathcal{I}}^\perp \mathcal{P}_{\Omega_A}\mathcal{P}_{K_\mathcal{I}}(\mathcal{P}_{K_\mathcal{I}}\mathcal{P}_{\Omega_A}\mathcal{P}_{K_\mathcal{I}})^{-1}\Vert \\
&= \sqrt{\frac{1}{\gamma_{\Omega_A^T} (K_\mathcal{I}^T)}-1} \textless \frac{1}{\alpha_K(L_0)} \\
\end{aligned}
\end{equation}
Take it into the equation (\ref{eq:proofalphanorm}), we finally have that
\begin{equation}
\label{eq:prooffinalyy}
\begin{aligned}
&\Vert \mathcal{P}_{K_\mathcal{I}}^\perp \mathcal{P}_{\Omega_A}(Y)K\Vert_{2,\infty} \textless 1
\end{aligned}
\end{equation}
This completes the proof of Theorem~\ref{thm:noiseless} in light of Lemma~\ref{lemma:noiseless}.$\hfill\blacksquare$

\subsection{Proof to Theorem \ref{thm:noisy}}

\textit{Proof:}\par
 Denote $N=\bar{L}-L_0$ and $N_A=\mathcal{A}_k(N)$. Knowing that $\Vert \mathcal{P}_{\Omega} (\bar{L}-M)\Vert_{F}\le \epsilon$ and $\Vert \mathcal{P}_{\Omega} (L_0-M)\Vert_{F}\le \epsilon$ and by triangle inequality, we have that $\Vert \mathcal{P}_{\Omega} (N)\Vert_{F}\le 2\epsilon$ which further means that $\Vert \mathcal{P}_{\Omega_A} (N_A)\Vert_{F}^2\le 4k\epsilon^2$.\par
 Now, we define $Y$ and $F_0$ in the same way as in the proof to Theorem 1. Notice that $\bar{L}=L_0+N$ is an optimal solution, we have the following:
 \begin{equation}
\begin{aligned}
0 &\ge\Vert \mathcal{A}_k(L_0+N)\Vert_{2,1}-\Vert \mathcal{A}_k(L_0)\Vert_{2,1}\\
&\ge (1-||F_0||_{2,\infty})||\mathcal{P}_{K_\mathcal{I}}^\perp (N_A) K||_{2,1}+\left<\mathcal{P}_{\Omega_A}(Y),N_AK\right>  \\
\end{aligned}
\end{equation}
Notice that for any matrix $X \in R^{m\times k}$, we have that $\Vert  X \Vert_{F}\le\Vert  X \Vert_{2,1}\le \sqrt{k}\Vert  X \Vert_{F}$ according to the fundamental inequality. Provided that $\rho_0 > 1-\frac{0.64k}{(0.64+\alpha_K^2(L_0))\mu_{K}^{\mathcal{I}}r_K(L_0)m}$, we have that $\Vert \mathcal{P}_{K_\mathcal{I}}^\perp \mathcal{P}_{\Omega_A}(Y)\Vert < 0.8$ as well as $||F_0||_{2,\infty} < 0.8$. By Lemma 2.5 in~\cite{7152948}, We have that

\begin{equation}
\begin{aligned}
&\Vert \mathcal{P}_{K_\mathcal{I}}^\perp (N_A) K\Vert_{2,1} \le -5\left<\mathcal{P}_{\Omega_A}(Y),N_AK\right>\\
&\le 5\Vert \mathcal{P}_{\Omega_A}(Y)K^T\Vert_{2,\infty}\Vert \mathcal{P}_{\Omega_A} (N_A)\Vert_{2,1}\\
&\le 5\Vert  \mathcal{P}_{\Omega_A}(Y)\Vert \Vert K^T\Vert_{2,\infty}\Vert \mathcal{P}_{\Omega_A} (N_A)\Vert_{2,1}\\
&\le 9\Vert \mathcal{P}_{\Omega_A} (N_A)\Vert_{2,1} \\
&\le 9\sqrt{k}\Vert \mathcal{P}_{\Omega_A} (N_A)\Vert_{F} \\
&\le 18k\epsilon
\end{aligned}
\end{equation}
then we have that $\Vert  \mathcal{P}_{K_\mathcal{I}}^\perp (N_A)\Vert_F=\Vert  \mathcal{P}_{K_\mathcal{I}}^\perp (N_A)K\Vert_F\le 18k\epsilon$. Hence, we have that
\begin{equation}
\begin{aligned}
&\Vert \mathcal{P}_{K_\mathcal{I}}^\perp \mathcal{P}_{\Omega_A}^\perp (N_A)\Vert_{F} \\
&= \Vert \mathcal{P}_{K_\mathcal{I}}^\perp (N_A)-\mathcal{P}_{K_\mathcal{I}}^\perp \mathcal{P}_{\Omega_A} (N_A)\Vert_{F}\\
&\le (18k+2\sqrt{k})\epsilon
\end{aligned}
\end{equation}
Then
\begin{equation}
\begin{aligned}
&\Vert \mathcal{P}_{\Omega_A}\mathcal{P}_{K_\mathcal{I}} \mathcal{P}_{\Omega_A}^\perp (N_A)\Vert_{F}^2 \\&
=\left<\mathcal{P}_{K_\mathcal{I}}\mathcal{P}_{\Omega_A}\mathcal{P}_{K_\mathcal{I}}\mathcal{P}_{\Omega_A}^\perp(N_A),\mathcal{P}_{K_\mathcal{I}}\mathcal{P}_{\Omega_A}^\perp(N_A)\right>\\
&=\left<\mathcal{P}_{K_\mathcal{I}}(I-\mathcal{P}_{\Omega_A}^\perp)\mathcal{P}_{K_\mathcal{I}}\mathcal{P}_{\Omega_A}^\perp(N_A),\mathcal{P}_{K_\mathcal{I}}\mathcal{P}_{\Omega_A}^\perp(N_A)\right>\\
&\ge (1-\Vert \mathcal{P}_{K_\mathcal{I}} \mathcal{P}_{\Omega_A}^\perp \mathcal{P}_{K_\mathcal{I}} \Vert)\Vert \mathcal{P}_{K_\mathcal{I}}\mathcal{P}_{\Omega_A}^\perp(N_A)\Vert_{F}^2\\
&\ge \frac{\alpha_K^2(L_0)}{1+ \alpha_K^2(L_0)}\Vert \mathcal{P}_{K_\mathcal{I}}\mathcal{P}_{\Omega_A}^\perp(N_A)\Vert_{F}^2
\end{aligned}
\end{equation}

which means that
\begin{equation}
\begin{aligned}
&\Vert \mathcal{P}_{K_\mathcal{I}}\mathcal{P}_{\Omega_A}^\perp(N_A)\Vert_{F}^2 \\&
\le\frac{1+\alpha_K^2(L_0)}{\alpha_K^2(L_0)}\Vert \mathcal{P}_{\Omega_A}\mathcal{P}_{K_\mathcal{I}} \mathcal{P}_{\Omega_A}^\perp (N_A)\Vert_{F}^2\\
&=\frac{1+\alpha_K^2(L_0)}{\alpha_K^2(L_0)}\Vert \mathcal{P}_{\Omega_A}\mathcal{P}_{K_\mathcal{I}}^\perp  \mathcal{P}_{\Omega_A}^\perp (N_A)\Vert_{F}^2\\
&\le \frac{1+\alpha_K^2(L_0)}{\alpha_K^2(L_0)}(18k+2\sqrt{k})^2\epsilon^2
\end{aligned}
\end{equation}
Finally, we have that
\begin{equation}
\begin{aligned}
&\Vert N\Vert_{F} =\Vert N_A\Vert_{F} /\sqrt{k}\le (\Vert  \mathcal{P}_{K_\mathcal{I}}^\perp (N_A)\Vert_F\\&+\Vert  \mathcal{P}_{K_\mathcal{I}} \mathcal{P}_{\Omega_A}^\perp(N_A) \Vert_F+\Vert  \mathcal{P}_{K_\mathcal{I}} \mathcal{P}_{\Omega_A}(N_A)\Vert_F)/\sqrt{k}\\
&\le(18\sqrt{k}+2)\frac{\alpha_K(L_0)+\sqrt{1+\alpha_K^2(L_0)}}{\alpha_K(L_0)}\epsilon\\
\end{aligned}
\end{equation}$\hfill\blacksquare$

\section{Experiment}\label{sec:Experiment}
 To verify the effectiveness and efficiency of ICNNM, we conduct experiments on several vision tasks, including image completion, video completion, video prediction and video frame interpolation. Besides CNNM, some representative methods in each task are also considered for comparison. However, Learning-Based CNNM (LbCNNM)~\cite{liu2022time}, which is also a learning-based extension to CNNM, is not included, as LbCNNM focuses on a different problem space and is made specific to vectors (i.e., order-$1$ tensors).\par 
The video frame interpolation experiments are conducted using PyTorch on an NVIDIA A800 GPU, while the remaining experiments are performed in MATLAB R2024a on an NVIDIA GeForce RTX 4090 GPU.
\begin{figure*}[t!]
    \centering
    \begin{subfigure}[b]{.15\linewidth}
        \centering
        \includegraphics[width=\linewidth]{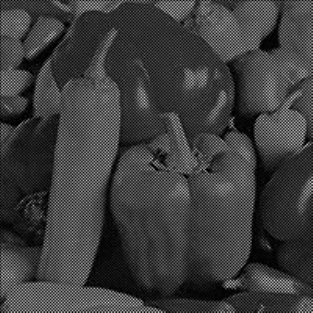}\\\vspace{0.01in}
        \includegraphics[width=\linewidth]{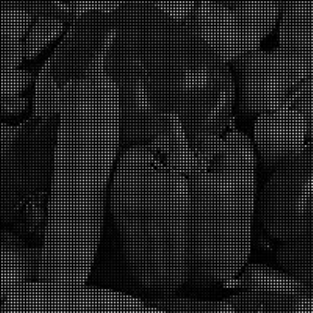}
        \caption{Input}
        \label{fig:irnm}
    \end{subfigure}%
    \hspace{0.02\linewidth}
    \begin{subfigure}[b]{.15\linewidth}
        \centering
        \includegraphics[width=\linewidth]{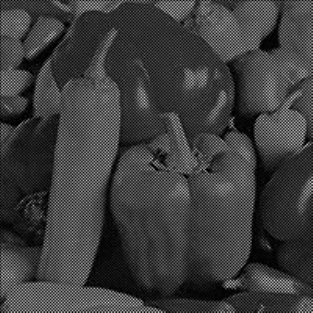} \\\vspace{0.01in}
        \includegraphics[width=\linewidth]{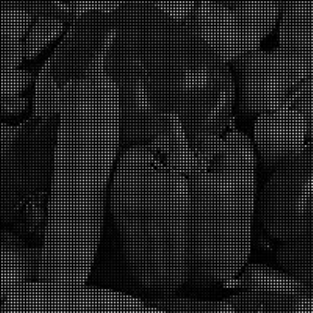}
        \caption{IRNN}
        \label{fig:irnm}
    \end{subfigure}%
    \hspace{0.02\linewidth}
    \begin{subfigure}[b]{.15\linewidth}
        \centering
        \includegraphics[width=\linewidth]{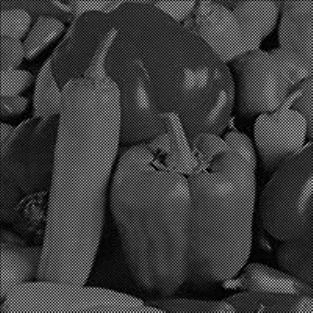} \\\vspace{0.01in}
        \includegraphics[width=\linewidth]{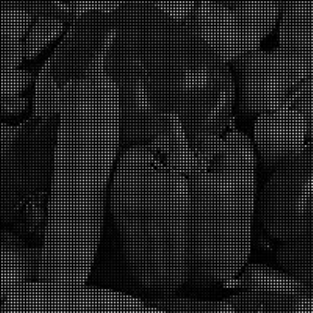}
        \caption{$\mathrm{DFT}_{\ell_1}$}
        \label{fig:dft}
    \end{subfigure}%
    \hspace{0.02\linewidth}
    \begin{subfigure}[b]{.15\linewidth}
        \centering
        \includegraphics[width=\linewidth]{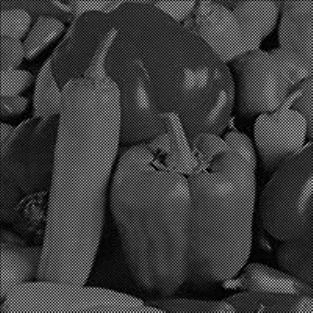} \\\vspace{0.01in}
        \includegraphics[width=\linewidth]{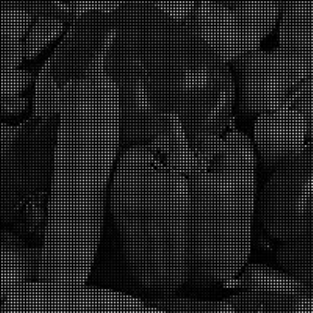}
        \caption{CNNM}
        \label{fig:cnnm}
    \end{subfigure}%
    \hspace{0.02\linewidth}
    \begin{subfigure}[b]{.15\linewidth}
        \centering
        \includegraphics[width=\linewidth]{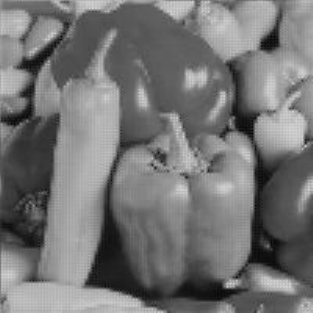} \\\vspace{0.01in}
        \includegraphics[width=\linewidth]{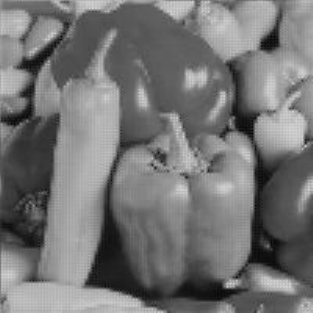}
        \caption{ICNNM}
        \label{fig:icnnm}
    \end{subfigure}%
    \medskip\vspace{-0.1in}
    \caption{
        An example from CSet8 (the rest 7 images are used for training). The missing rates (i.e., $1-\rho_0$) for the first and second rows are $50\%$ and $75\%$, respectively. For ICNNM and CNNM, the kernel size is set as $13 \times 13$. While the competing methods all fail, ICNNM produces superior results with PSNR being $26.42$ and $23.11$ under two missing rates, respectively.  
    }
    \label{fig:example}
\end{figure*}
\subsection{Image Completion}
We first consider the task of restoring the missing pixels in images (i.e., order-2 tensors). Two public datasets are used for experiments, including CSet8 (8 images) and Kodak (24 images, \url{https://r0k.us/graphics/kodak/}). All images are resized to the resolution of $200 \times 200$,
and 7 images from CSet8 are used as reference images for learning the convolution eigenvectors required by ICNNM (the rest 1 image is for testing). To examine the generalization ability of various methods, all 24 images from Kodak are regarded as the testing data. The competing methods include IRNN~\cite{IRNN} as a representative of traditional low rank methods, $\mathrm{DFT}_{\ell_1}$ and CNNM~\cite{liu2022recovery}. The kernel size is set to $13 \times 13$ for both ICNNM and CNNM.  
\begin{figure}[H]
\centering
\includegraphics[width=0.75\linewidth]{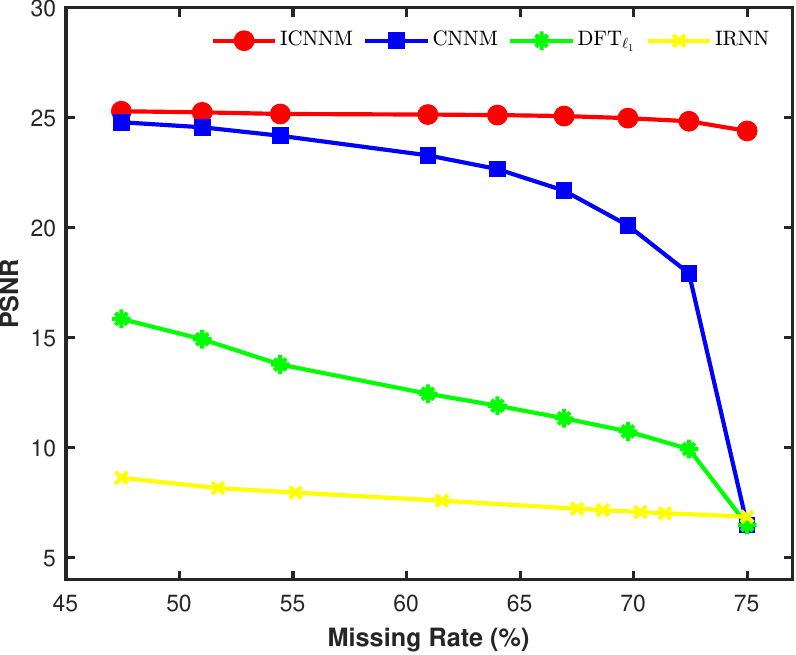}
\caption{Performance of image completion averaged on Kodak.}
\label{fig:Kodak}
\end{figure}
The sampling pattern is specifically designed such that the adjacent pixels are likely to be wholly missing, as exhibited in the first column of Fig~\ref{fig:example}. As we can see from Fig~\ref{fig:example}, even in the difficult cases where all the competing methods fail, ICNNM still achieves fairly accurate recovery. Fig~\ref{fig:Kodak} presents the overall performance of various methods on Kodak (24 images), demonstrating that ICNNM consistently outperforms the competing methods, especially when the missing rate is high. These results also illustrate that ICNNM owns strong generalization ability, as the convolution eigenvectors are learned from a different dataset, CSet8. 
\begin{table}[h]
\centering
\caption{Performance of ICNNM under different schemes of constructing reference images. Results are averaged on Kodak.}
\label{tab:Generalization}
\begin{tabular}{@{}lll@{}}
\toprule
Scheme & $\alpha_K(L_0)$ & PSNR \\
\midrule
G.T. & 1 & 24.62 \\
One randomly generated matrix & 4.5995 & 19.68 \\
The Lena's image & 1.6348 & 24.04 \\
7 images from CSet8 & \textbf{1.4998} & \textbf{24.41} \\
\bottomrule
\end{tabular}
\end{table}

To further investigate the influence of reference images, we test four different schemes: 1) using the target $L_0$ as the reference image, denoted as ``G.T''; 2)using a randomly generated matrix; 3) using the Lena's image public available on the Internet; 4) using 7 images from CSet8. The comparison results are shown in Table~\ref{tab:Generalization}. As we can see, while the reference images are chosen ``carelessly'' from the domain of nature images, the spectral correlation coefficient $\alpha_K(L_0)$ is close to 1 and therefore the recovery performance of ICNNM can get very close to the ideal scheme of using the target $L_0$ for pre-training. These results also illustrate that the convolution spectrum of natural images may be mostly similar, providing evidence to support the doctrine that Convolution Neural Networks have good generalization ability in the natural image domain. 
\begin{table}[h]
\centering
\caption{Performance comparison on the Kodak under deterministic sampling with a missing rate of $60\%$.}
\label{tab:kodak_deterministic_60}
\begin{tabular}{llcc}
\toprule
Method & Kernel Size & PSNR (dB) & Time (s) \\
\midrule
IRNN & \multicolumn{1}{c}{--} & 8.837 & 0.979 \\
$\mathrm{DFT}_{\ell_1}$ & \multicolumn{1}{c}{--} & 14.719 & \textbf{0.341} \\
\midrule
\multirow{4}{*}{CNNM}
& $13 \times 13$ & 27.499 & 18.643 \\
& $23 \times 23$ & 27.604 & 116.186 \\
& $33 \times 33$ & 27.618 & 761.419 \\
& $43 \times 43$ & 27.608 & 2056.302 \\
\midrule
\multirow{4}{*}{ICNNM}
& $13 \times 13$ & 27.811 & 11.189 \\
& $23 \times 23$ & \textbf{27.918} & 59.676 \\
& $33 \times 33$ & 27.908 & 372.174 \\
& $43 \times 43$ & 27.757 & 1061.474 \\
\bottomrule
\end{tabular}
\end{table}

Table~\ref{tab:kodak_deterministic_60} reports the reconstruction performance in terms of PSNR and the average computational time (in seconds) of IRNN, $\mathrm{DFT}_{\ell_1}$, as well as CNNM and ICNNM with different kernel sizes, under deterministic sampling with a missing rate of $60\%$ on the Kodak dataset. The results demonstrate that both CNNM and ICNNM achieve substantially higher reconstruction accuracy than the competing methods. Moreover, compared with CNNM, ICNNM consistently yields superior PSNR while requiring less computational time. This computational advantage becomes increasingly pronounced as the convolutional kernel size grows.

\subsection{Video Completion}\label{subsec:vp}
 We now consider the task of completing partially observed videos, where some frames are chosen randomly to be missing. For experiments, we select two videos from the CDNet 2014 dataset~\cite{6910011}, one of which is used for training and the other for testing. These two videos both have 31 $50\times50$ frames, but record different scenes at different time. Five methods are taken for comparison, including HaLRTC~\cite{6138863}, TNNM~\cite{TNN}, $\mathrm{DFT}_{\ell_1}$, CNNM and WA, where WA is a naive method that restores a missing frame by simply using the weighted average of its adjacent, observed frames. For CNNM and ICNNM, the kernel size is set as $13\times13\times13$.

\begin{figure}[ht]
    \centering

    \begin{minipage}{\linewidth}
        \centering
        \includegraphics[width=\linewidth]{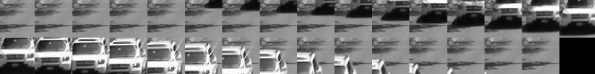}
        \caption*{Training frames}
    \end{minipage}

    \vspace{0.5em}

    \begin{minipage}{\linewidth}
        \centering
        \includegraphics[width=\linewidth]{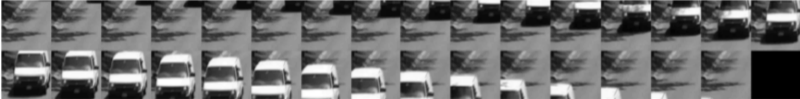}
        \caption*{Test frames}
    \end{minipage}

    \caption{
        Visualization of training and test frames used in the task of video completion and video prediction.
        Each frame has a spatial size of $50 \times 50$. The training frames are only used in ICNNM.
    }
\end{figure}

Fig~\ref{fig:video} presents the results of various methods. Since the sampling pattern is of frame-wise, the methods based on Tucker low-rankness, e.g., HaLRTC and TNNM, will use zero to fill in the missing entries and therefore perform poorly. By contrast, convolutional low-rankness based methods, e.g., CNNM, $\mathrm{DFT}_{\ell_1}$ and ICNNM, can handle such cases. Yet, as shown in Fig~\ref{fig:video}, our ICNNM outperforms distinctly the most close competing method, CNNM. \par
\begin{figure}[H]
\centering
\includegraphics[width=0.75\linewidth]{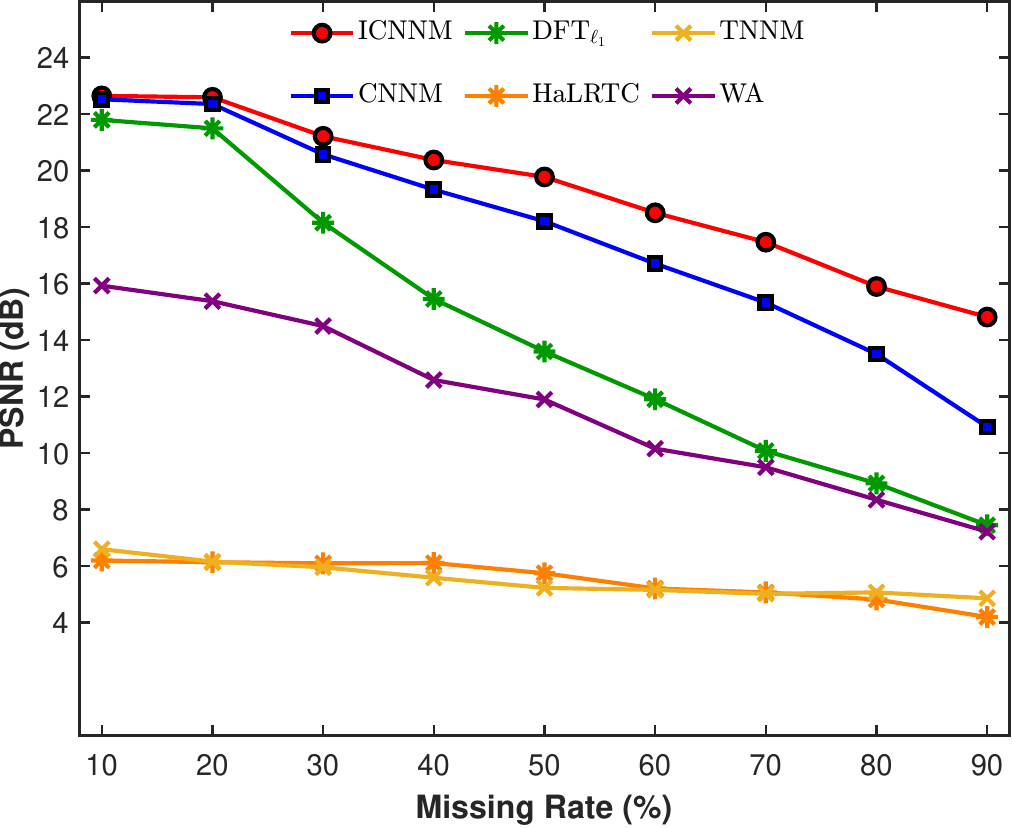}
\caption{Performance of various methods in video completion averaged from 20 runs.}
\label{fig:video}
\end{figure}
Fig~\ref{fig:time} compares the runtime of CNNM and ICNNM, under different settings for the kernel size $k_1\times{}k_2\times{}k_3$. As can be seen, ICNNM is much faster than CNNM, especially when the kernel size is large. This is because, when the matrices getting larger, the advantage of matrix multiplication over SVD will be more dramatic on GPU. 
\begin{figure}[H]
\centering
\includegraphics[width=0.75\linewidth]{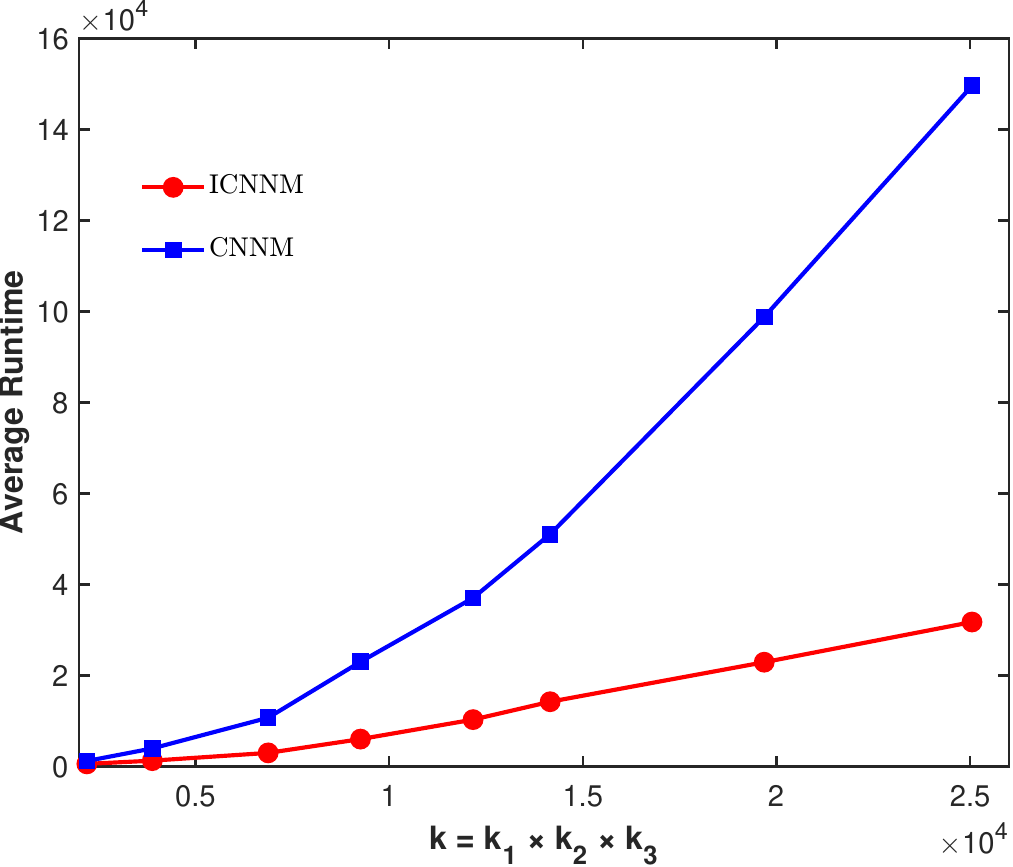}
\caption{Runtime of CNNM and ICNNM under different kernel sizes averaged from 20 runs.}
\label{fig:time}
\end{figure}

\subsection{Video Prediction}\label{subsec:vp}
Furthermore, we conducted experiments under a more extreme sampling method, known as video prediction: the task is to predict the last 8 frames based on the former 23 frames. Besides the baselines considered in video completion, we also include a deep learning based method, termed ConvLSTM~\cite{convlstm}, for comparison. Unlike the proposed ICNNM, which uses only one video for training, the ConvLSTM here is trained using 60 videos of dimension $50\times50\times31$ from CDNet. The first two values of the kernel size are set as $k_1=k_2=13$ for both CNNM and ICNNM. As for $k_3$ that corresponds to the time dimension, as pointed out by~\cite{liu2022recovery}, relatively large numbers are often desirable. \par
\begin{table}[ht]
\centering
\setlength{\tabcolsep}{2.5pt} 
\caption{Performance (in terms of PSNR) of various methods in the task of predicting the last 8 frames based on the first 23 frames of a $50\times50\times31$ video. The number in parentheses after CNNM represents the value of $k_3$, and so as ICNNM.}
\label{tab:prediction}
\begin{tabular}{@{}l
                S[table-format=2.2]
                S[table-format=2.1]
                S[table-format=2.1]
                S[table-format=2.1]
                S[table-format=2.1]
                S[table-format=2.1]
                S[table-format=2.1]
                S[table-format=2.1]@{}}
\toprule
\multirow{2}{*}{\normalsize\shortstack{Method}} & 
\multicolumn{8}{c}{\normalsize Frame} \\ 
\cmidrule(l{1.5pt}r{1.5pt}){2-9}
& 1 & 2 & 3 & 4 & 5 & 6 & 7 & 8 \\ 
\midrule
HaLRTC & 4.20 & 4.8 & 5.1 & 5.0 & 5.0 & 5.2 & 5.8 & 6.1 \\ 
TNNM   & 4.20 & 4.8 & 5.1 & 5.0 & 5.0 & 5.2 & 5.8 & 6.1 \\
\midrule
ConvLSTM & 15.67 & 11.4 & 10.1 & 9.6 & 9.9 & 10.6 & 12.3 & 19.0 \\
$\mathrm{DFT}_{\ell_1}$ & 15.89 & 11.7 & 9.8 & 8.9 & 8.5 & 8.8 & 11.6 & 18.2 \\
CNNM (13) & 16.51 & 13.6 & 12.6 & 12.3 & 12.5 & 13.2 & 16.1 & 21.9 \\
CNNM (23) & 16.78 & 13.8 & 12.6 & 12.3 & 12.5 & 13.3 & 16.3 & 22.4 \\
ICNNM (13) & 16.73 & 14.0 & 13.0 & 12.8 & 13.0 & 13.7 & 16.5 & 22.1 \\
ICNNM (23) & \textbf{17.13} & \textbf{14.3} & \textbf{13.2} & \textbf{13.0} & \textbf{13.3} & \textbf{14.1} & \textbf{17.1} & \textbf{22.9} \\
\bottomrule
\end{tabular}
\end{table}
Table~\ref{tab:prediction} shows the PSNR values achieved by various methods. Due to the same reason as the case of video completion, HaLRTC and TNNM predict the missing frames as zeros, leading to very low PSNR values. As shown in Table~\ref{tab:prediction}, ICNNM consistently outperforms the competing methods across different time stamps, confirming again the superiority of using pre-learned convolution eigenvectors to guide the recovery process. \par
\begin{figure*}[h]
    \centering

    \begin{subfigure}[b]{.16\linewidth}
        \centering
        \includegraphics[width=\linewidth]{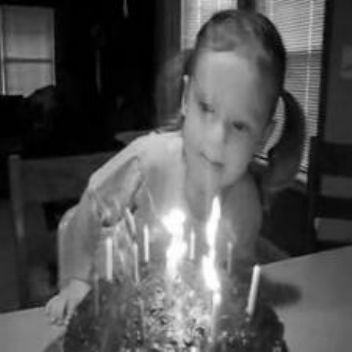} \\
        \includegraphics[width=\linewidth]{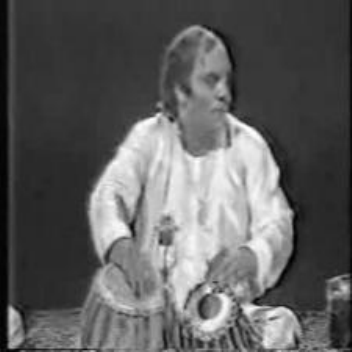}
         \includegraphics[width=\linewidth]{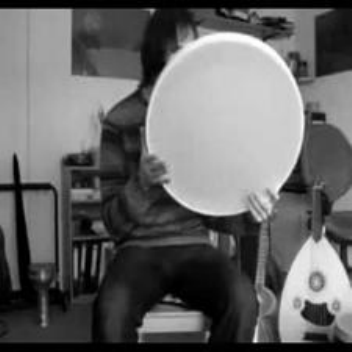}
        \caption{frame 0}
    \end{subfigure}%
    \hfill
    \begin{subfigure}[b]{.16\linewidth}
        \centering
        \includegraphics[width=\linewidth]{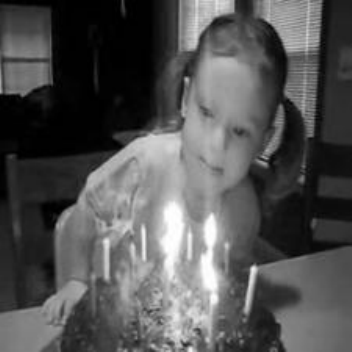} \\
        \includegraphics[width=\linewidth]{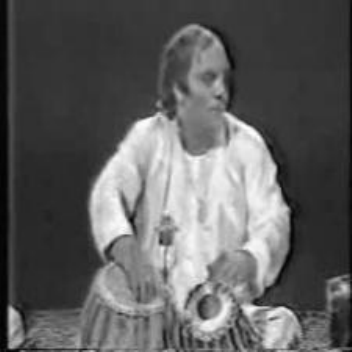}
         \includegraphics[width=\linewidth]{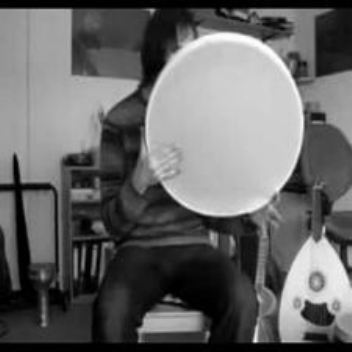}
        \caption{frame 1}
    \end{subfigure}%
    \hfill
    \begin{subfigure}[b]{.16\linewidth}
        \centering
        \includegraphics[width=\linewidth]{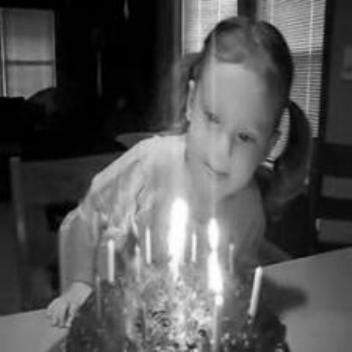} \\
        \includegraphics[width=\linewidth]{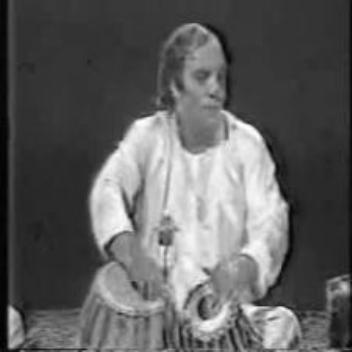}
         \includegraphics[width=\linewidth]{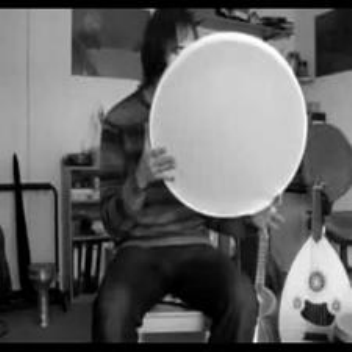}
        \caption{frame 2}
    \end{subfigure}%
    \hfill
    \begin{subfigure}[b]{.16\linewidth}
        \centering
        \includegraphics[width=\linewidth]{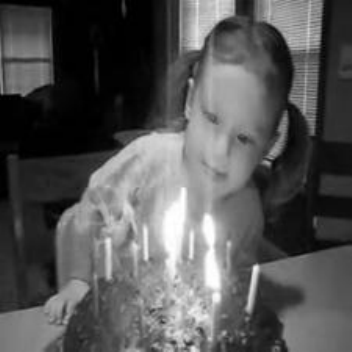} \\
        \includegraphics[width=\linewidth]{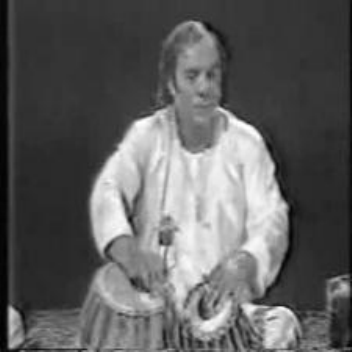}
         \includegraphics[width=\linewidth]{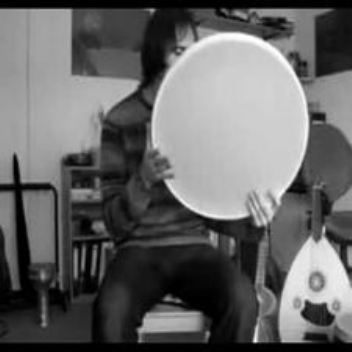}
        \caption{frame 3}
    \end{subfigure}%
    \hfill
    \begin{subfigure}[b]{.16\linewidth}
        \centering
        \includegraphics[width=\linewidth]{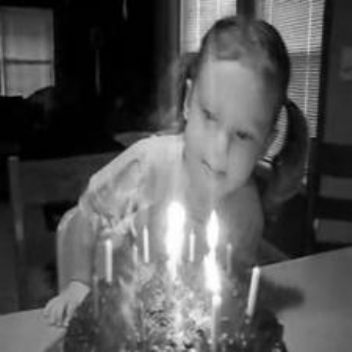} \\
        \includegraphics[width=\linewidth]{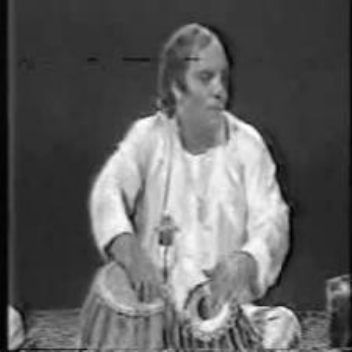}
         \includegraphics[width=\linewidth]{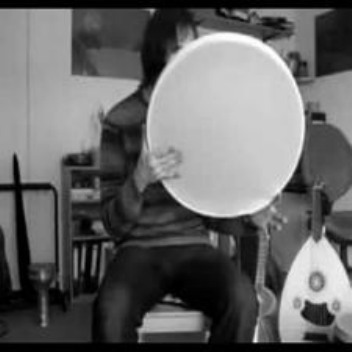}
        \caption{Ground Truth}
    \end{subfigure}%
    \hfill
    \begin{subfigure}[b]{.16\linewidth}
        \centering
        \includegraphics[width=\linewidth]{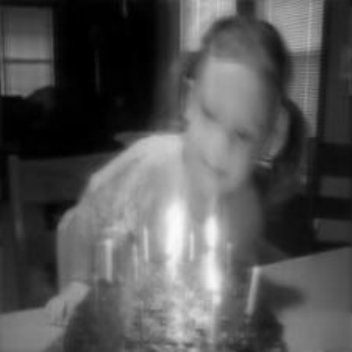} \\
        \includegraphics[width=\linewidth]{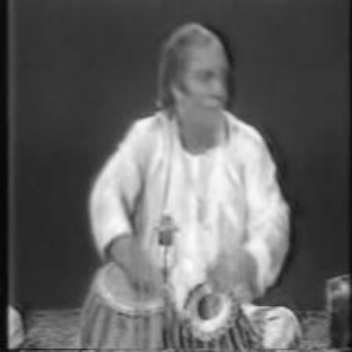}
         \includegraphics[width=\linewidth]{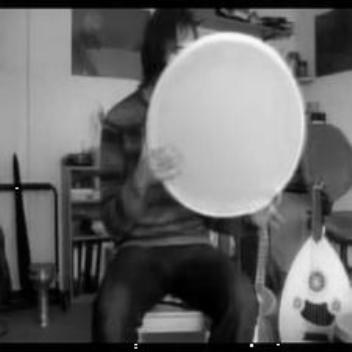}
        \caption{ICNNM}
    \end{subfigure}%
    \medskip
 
    \caption{
        Results of video interpolation with 4 reference frames. The dataset is UCF-101 processed by ~\cite{QVI}. "Ground Truth" denotes the ground truth of the target frame in the middle of "frame 1" and "frame 2". We convert all images to gray scale images and select smaller kernel size of $27\times 27\times 3$ in the consideration of memory overhead.
    }
    \label{fig:Supplementary3}
\end{figure*}

\subsection{Video Frame Interpolation}
Video frame interpolation~\cite{Jiang2017SuperSH,Huang2020RealTimeIF,Zhang2023ExtractingMA} is a popular topic in vision. Currently, most mainstream methods rely on optical flow estimation and inference, achieving high-quality frame interpolation results. However, the optical flow method is based on the assumption that the motion between two frames is small. This assumption may not hold in practical scenarios, such as when the sampling rate of the input video is very low, leading to large time intervals between adjacent frames and potentially compromising the effectiveness of the optical flow method. In contrast, our ICNNM may perform well in such situations, as it implicitly learns the temporal relationships of the target frames from reference tensors.

We use the UCF-101 dataset~\cite{Soomro2012UCF101AD} for experiments. We compare ICNNM with CNNM and several state-of-the-art methods, including SuperSloMo~\cite{Jiang2017SuperSH}, RIFE~\cite{Huang2020RealTimeIF} and EMA-VFI~\cite{Zhang2023ExtractingMA}. Since these competing methods are all specific to the setting of inserting one frame between two frames, the videos in UCF-101 are split into $256 \times 256\times3$, and the kernel size of ICNNM and CNNM is set as $64 \times 64 \times3$.  

Table~\ref{tab:interpolation} reports the quantitative results in terms of PSNR, SSIM, and average computational time. The results indicate that ICNNM consistently outperforms CNNM in both reconstruction accuracy and computational efficiency. In comparison with the competing methods, ICNNM achieves the highest PSNR while delivering comparable SSIM performance. These results demonstrate that, despite being a generic completion approach, ICNNM attains competitive performance even when compared with state-of-the-art methods specifically designed for video frame interpolation.
\begin{table}[h]
\centering
\caption{Comparison against state-of-the-art methods on UCF-101, under the context of video frame interpolation.}
\label{tab:interpolation}
\begin{tabular}{@{}llll@{}}
\toprule
Method & PSNR & SSIM &Time(s)\\
\midrule
CNNM & 21.57 & 0.7732 &4971.51\\
SuperSloMo~\cite{Jiang2017SuperSH} & 25.42 & 0.8773 &\multicolumn{1}{c}{--}\\
RIFE~\cite{Huang2020RealTimeIF} & 29.88 & 0.9391 &\multicolumn{1}{c}{--}\\
EMA-VFI~\cite{Zhang2023ExtractingMA} & 30.16 & \textbf{0.9409} &\multicolumn{1}{c}{--}\\
ICNNM & \textbf{30.61} & 0.9282 &716.34\\
\bottomrule
\end{tabular}
\end{table}

\section{Conclusion and Future Work}
\label{sec:Conclusion}
TCAS which aims to restore a tensor from a subset of its entries sampled arbitrarily, has considerable significance in a wide range of applications and can be solved by CNNM. To resolve issues of high computational cost and sampling complexity of CNNM, we proposed in this paper a novel method termed ICNNM, which utilizes a set of pre-learned convolution eigenvectors to guide the recovery process. The recovery performance of ICNNM was analyzed theoretically, and verified via extension experiments on various tasks, including image completion, video completion, video predicting and video frame interpolation. Both the theoretical analysis and experimental results verify the superiority of ICNNM over existing methods. \par
The computational procedure of ICNNM can be unfolded to a deep network, the structure of which is very similar to the prevalent Convolution Neural Networks. However, how to utilize this structure effectually in still unknown. We shall leave this as the future work.

\bibliography{ref}

\begin{IEEEbiography}
[{\includegraphics[width=1in,
height=1.25in,clip, keepaspectratio]{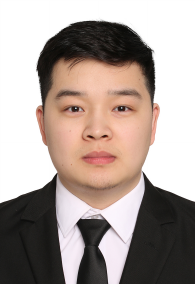}}]{Wei Li} (S’25) is working toward the Ph.D.
degree in the School of Automation, Southeast University, China. He has received the bachelor’s degree in engineering from Southeast University, China, in 2022. His research interests include tensor completion, image processing and generation. He is a student member of the IEEE.
\end{IEEEbiography}
\begin{IEEEbiography}
[{\includegraphics[width=1in,
height=1.25in,clip, keepaspectratio]{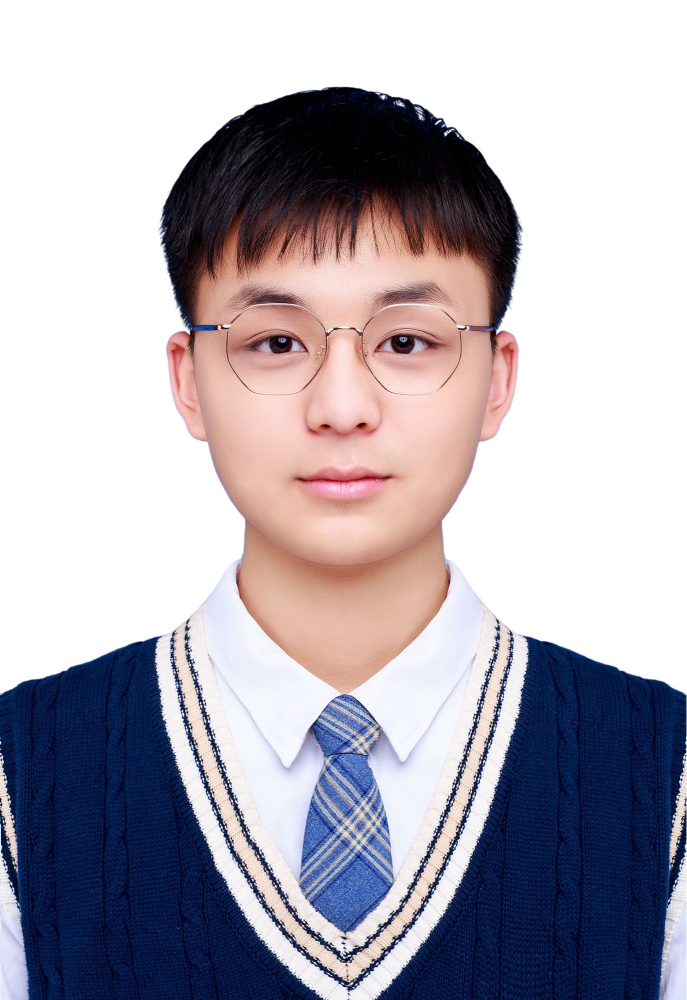}}]{Yuyang Li}(S'25)  is working toward the Ph.D. degree in the School of Automation, Southeast University, China. 
He has received the bachelor's degree in engineering from Southeast University, China, in 2022.
His research interests touch on the areas of machine learning, computer vision and signal processing.
He is a student member of the IEEE.
\end{IEEEbiography}

\begin{IEEEbiography}
[{\includegraphics[width=1in,
height=1.25in,clip, keepaspectratio]{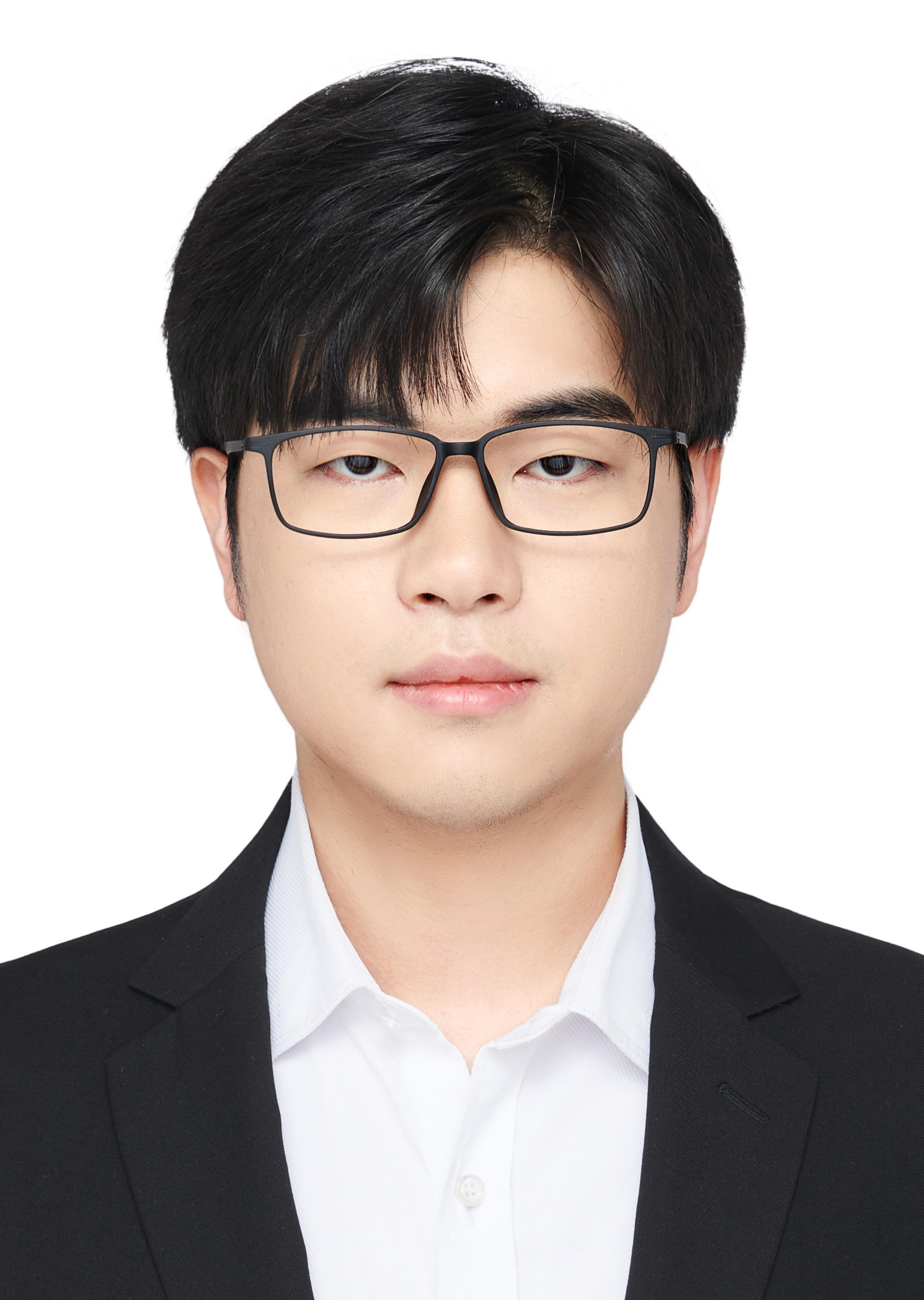}}]{Kaile Du} (S'23)  is working toward the Ph.D. degree in the School of Automation, Southeast University, China. He has published several academic papers on international journals and conferences, such as IEEE TRANSACTIONS ON MULTIMEDIA, IEEE TRANSACTIONS ON
GEOSCIENCE AND REMOTE SENSING, Mathematics, ECCV, AAAI and ICME, etc. His research interests include class-incremental learning, multi-label learning, graph representation learning, etc.
\end{IEEEbiography}

\begin{IEEEbiography}
[{\includegraphics[width=1in,
height=1.25in,clip, keepaspectratio]{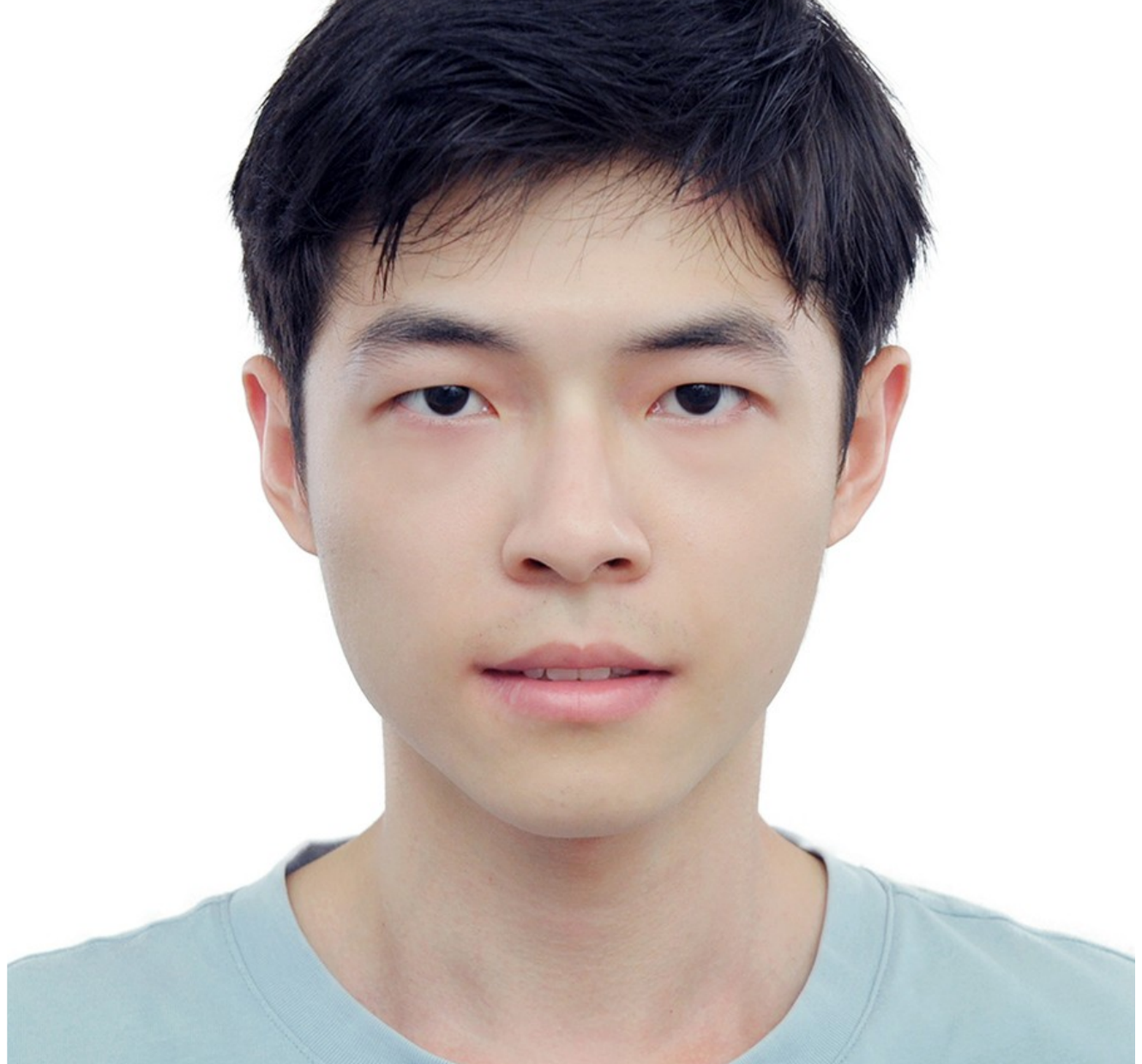}}]{Yi Yu}(S'24)  is working toward the master's degree in the School of Automation, Southeast University, China. He has received the bachelor's degree in engineering from Southeast University, China, in 2023. His research interests touch on the areas of machine learning, time series prediction and image restoration.
He is a student member of the IEEE.
\end{IEEEbiography}
\begin{IEEEbiography}
[{\includegraphics[width=1in,
height=1.25in,clip, keepaspectratio]{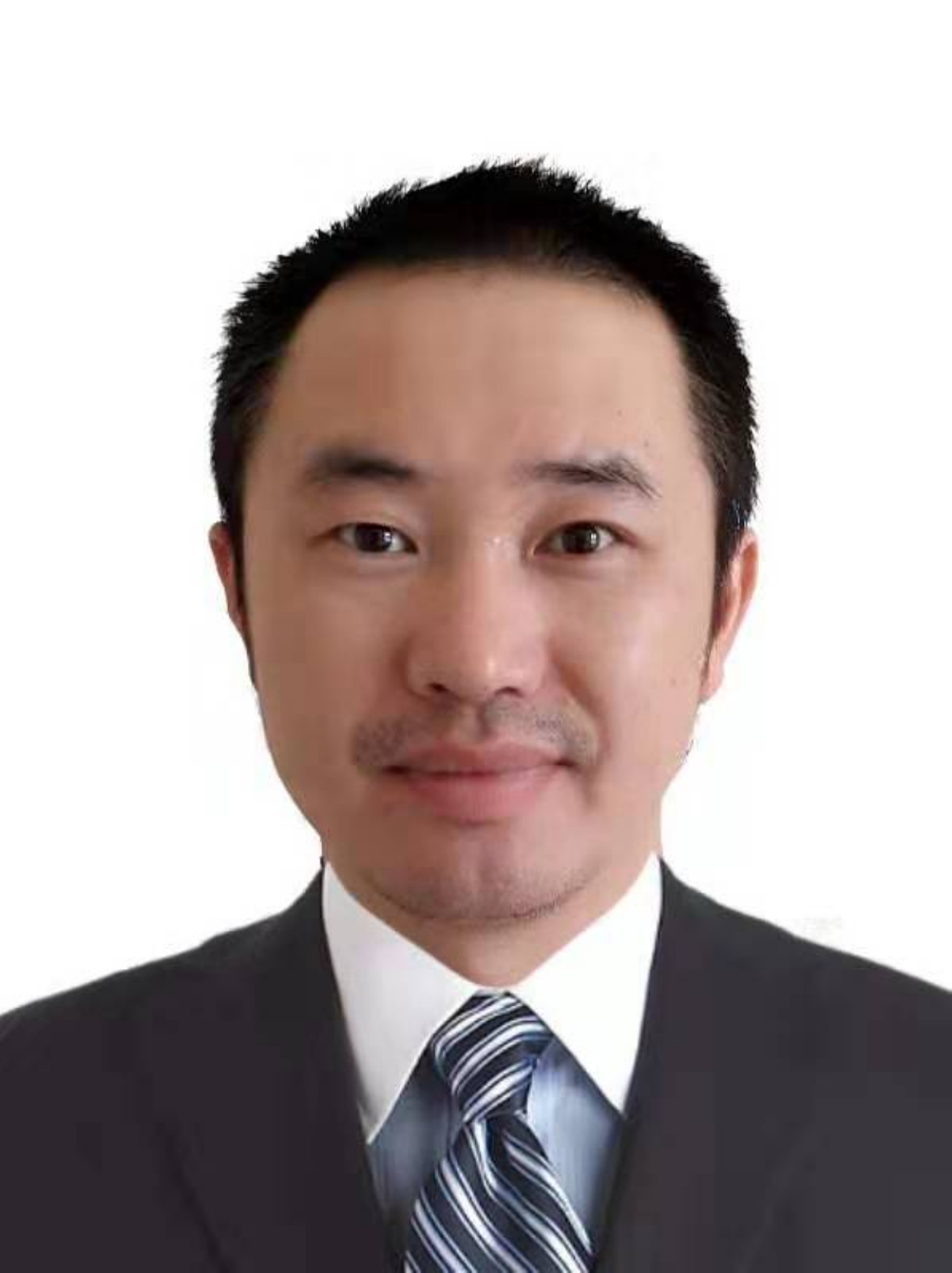}}]{Guangcan Liu}(M'11-SM'17) is currently a professor with the School of Automation, Southeast University, Nanjing, China. He received the bachelor's degree in mathematics and the Ph.D. degree in computer science and engineering from Shanghai Jiao Tong University, Shanghai, China, in 2004 and 2010, respectively. He was a Post-Doctoral Researcher with the National University of Singapore, Singapore, from 2011 to 2012, the University of Illinois at Urbana-Champaign, Champaign, IL, USA, from 2012 to 2013, Cornell University, Ithaca, NY, USA, from 2013 to 2014, and Rutgers University, Piscataway, NJ, USA, in 2014. He was a professor with the School of Automation, Nanjing University of Information Science and Technology, Nanjing, China, from 2014 to 2021. His research interests touch on the areas of machine learning, computer vision and signal processing. He is a senior member of the IEEE.
\end{IEEEbiography}

\end{document}